\documentclass[paper=letter, fontsize=20pt]{article}
\usepackage[square]{natbib}
\usepackage[english]{babel} 
\usepackage{amsmath,amsfonts,amsthm} 
\usepackage[utf8]{inputenc}
\usepackage{float}
\usepackage{lipsum} 
\usepackage{multirow}
\usepackage{array}
\usepackage{blindtext}
\usepackage{graphicx,subfigure}
\usepackage{caption}
\usepackage{appendix}
\makeatletter
\def\@seccntformat#1{\@ifundefined{#1@cntformat}%
   {\csname the#1\endcsname\quad}  
   {\csname #1@cntformat\endcsname}
}
\let\oldappendix\appendix 
\renewcommand\appendix{%
    \oldappendix
    \newcommand{\section@cntformat}{\appendixname~\thesection\quad}
}
\makeatother

\usepackage[sc]{mathpazo} 
\usepackage[T1]{fontenc} 
\usepackage{microtype} 
\usepackage[hmarginratio=1:1,top=32mm,columnsep=20pt]{geometry} 
\usepackage{multicol} 
\usepackage{booktabs} 
\usepackage{float} 
\usepackage{hyperref} 
\usepackage{lettrine} 
\usepackage{paralist} 
\usepackage{abstract} 
\usepackage{titlesec} 

\usepackage{algorithm}
\usepackage{algorithmic}

\usepackage{enumitem}
\DeclareMathOperator*{\argmax}{arg\,max}
\DeclareMathOperator*{\argmin}{arg\,min}
\DeclareMathOperator*{\esssup}{ess\,sup}
\DeclareMathOperator{\Var}{Var}
\DeclareMathOperator{\Mod}{\ \mathrm{mod}}

\theoremstyle{definition}
\newtheorem{myDef}{Definition}
\newtheorem{myThm}{Theorem} 
 
\newtheorem{myLemma}{Lemma} 

\newtheorem{myCor}{Corollary}
\newtheorem{myRemark}{Remark}

\newtheorem{myObservation}{Observation}

\title{\vspace{-7mm}\fontsize{17pt}{10pt}\selectfont\textbf{Handling Concept Drift via Model Reuse}} 
\author{
Peng Zhao, Le-Wen Cai, and Zhi-Hua Zhou\thanks{Corresponding author. Email: zhouzh@nju.edu.cn}\\ 
\small
National Key Laboratory for Novel Software Technology,\\
\small
Nanjing University, Nanjing 210023, China\\
\small
\{zhaop, cailw, zhouzh\}@lamda.nju.edu.cn\\
}
\date{}

\begin{document}
\maketitle

\hrule
\begin{abstract} 
In many real-world applications, data are often collected in the form of stream, and thus the distribution usually changes in nature, which is referred as \emph{concept drift} in literature. We propose a novel and effective approach to handle concept drift via model reuse, leveraging previous knowledge by reusing models. Each model is associated with a weight representing its reusability towards current data, and the weight is adaptively adjusted according to the model performance. We provide generalization and regret analysis. Experimental results also validate the superiority of our approach on both synthetic and real-world datasets.
\end{abstract} 
\hrule
\section{Introduction}

With a rapid development in data collection technology, it is of great importance to analyze and extract knowledge from them. However, data are commonly in a streaming form and are usually collected from non-stationary environments, and thus they are evolving in nature. In other words, the joint distribution between the input feature and the target label will change, which is also referred as \emph{concept drift} in literature~\citep{journals/csur/GamaZBPB14}. If we simply ignore the distribution change when learning from the evolving data stream, the performance will dramatically drop down, which are not empirically and theoretically suitable for these tasks. The concept drift problem has become one of the most challenging issues for data stream learning. It has gradually drawn researchers' attention to design effective and theoretically sound algorithms.

Data stream with concept drift is essentially almost impossible to learn (predict) if there is not any assumption on distribution change. That is, if the underlying distribution changes arbitrarily or even adversarially, there is no hope to learn a good model to make the prediction. We share the same assumption with most of the previous work, that is, \emph{there contains some useful knowledge for future prediction in previous data}. No matter sliding window based approaches~\citep{conf/icml/KlinkenbergJ00,conf/sdm/BifetG07,journals/ida/KunchevaZ09}, forgetting based approaches~\citep{conf/ecai/koychev2000gradual,journals/ida/Klinkenberg04} or ensemble based approaches~\citep{conf/icml/KolterM05,journals/jmlr/KolterM07,journals/tnnls/suny18}, they share the same assumption, whereas the only difference is how to utilize previous knowledge or data. 

Another issue is that most previous work on handling concept drift focus on the algorithm design, only a few work consider the theoretical part~\citep{journals/ml/HelmboldL94,conf/colt/CrammerMEV10,conf/alt/MohriM12}. There are some work proposing algorithms along with theoretical analysis, for example, \cite{conf/icml/KolterM05} provides mistake and loss bounds and guarantees that the performance of the proposed approach is relative to the performance of the base learner. \cite{conf/icml/HarelMEC14} detects concept drift via resampling and provides the bounds on differentiates based on stability analysis. However, seldom have clear theoretical guarantees, or justifications on why and how to leverage previous knowledge to fight with concept drift, especially from the generalization aspect.

In this paper, we propose a novel and effective approach for handling \underline{\textsc{Con}}cept \underline{\textsc{d}}rift via m\underline{\textsc{o}}del \underline{\textsc{r}}euse, or \textsc{Condor}. It consists of two modules, $\mathtt{ModelUpdate}$ module aims at leveraging previous knowledge to help build the new model and update model pool, while $\mathtt{WeightUpdate}$ module adaptively assigns the weights for previous models according to their performance, representing the reusability towards current data. We justify the advantage of $\mathtt{ModelUpdate}$ from the aspect of generalization analysis, showing that our approach can benefit from a good weighted combination of previous models. Meanwhile, the $\mathtt{WeightUpdate}$ module guarantees that the weights will concentrate on the better-fit models. Besides, we also provide the dynamic regret analysis. Empirical experiments on both synthetic and real-world datasets validate the effectiveness of our approach.

In the following, Section~\ref{sec:related-work} discusses related work. Section~\ref{sec:proposed-approach} proposes our approach. Section~\ref{sec:theoretical-analysis} presents theoretical analysis. Section~\ref{sec:experiments} reports the experimental results. Finally, we conclude the paper and discuss future work in Section~\ref{sec:conclusion}.

\section{Related Work}
\label{sec:related-work}
\textsc{Concept Drift} has been well-recognized in recent researches~\citep{journals/csur/GamaZBPB14,journals/csur/GomesBEB17}. Basically, if there is not any structural information about data stream, and the distribution can change arbitrarily or even adversarially, we shall not expect to learn from historical data and make any meaningful prediction. Thus, it is crucial to make assumptions about the concept drift stream. Typically, most previous work assume that the nearby data items contain more useful information w.r.t. the current data, and thus researchers propose plenty of approaches based on the \textit{sliding window} and \textit{forgetting} mechanisms. Sliding window based approaches maintain the nearest data items and discard old items, with a fixed or adaptive window size~\citep{conf/icml/KlinkenbergJ00,journals/ida/KunchevaZ09}. Forgetting based approaches do not explicitly discard old items but downweight previous data items according to their age~\citep{conf/ecai/koychev2000gradual,journals/ida/Klinkenberg04}. Another important category falls into the \textit{ensemble} based approaches, as they can adaptively add or delete base classifiers and dynamically adjust weights when dealing with evolving data stream. A series work borrows the idea from boosting~\citep{journals/ml/Schapire90} and online boosting~\cite{conf/icml/BeygelzimerKL15}, dynamically adjust weights of classifiers. Take a few representatives, dynamic weighted majority ($\mathtt{DWM}$) dynamically creates and removes weighted experts in response to changes~\citep{conf/icdm/KolterM03,journals/jmlr/KolterM07}. Additive expert ensemble ($\mathtt{AddExp}$) maintains and dynamically adjusts the additive expert pool, and provides the theoretical guarantee with solid mistake and loss bounds~\citep{conf/icml/KolterM05}. Learning in the non-stationary environments ($\mathtt{Learn}^{\mathtt{++}}$.$\mathtt{NSE}$) trains one new classifier for each batch of data it receives, and combines these classifiers~\citep{journals/tnn/ElwellP11}. There are plenty of approaches to learning or mining from the evolving data stream, readers can refer to a comprehensive survey~\citep{journals/csur/GamaZBPB14,journals/csur/GomesBEB17}. As for boosting and ensemble approaches, readers are recommended to read the books~\citep{book/MIT/Schapire2012,book/Chapman/zhou2012}. 

Our approach is kind of similar to DWM and AddExp on the surface. We all maintain a model pool and adjust weights to penalty models with poor performance. However, we differ from the model update procedure and they ignore to leverage previous knowledge and reuse models to help build new model and  update model pool. Besides, our weight update strategies are also different. 

\textsc{Model Reuse} is an important learning problem, also named as model transfer, hypothesis transfer learning, or learning from auxiliary classifiers. The basic setting is that one desires to reuse pre-trained models to help further model building, especially when the data are too scarce to directly train a fair model. A series work lies in the idea of \emph{biased regularization}, which leverages previous models as the bias regularizer into empirical risk minimization, and achieves a good performance in plenty of scenarios~\citep{conf/icml/DuanTXC09,conf/cvpr/TommasiOC10,journals/pami/TommasiOC14}. There are also some other attempts and applications like model reuse by random forests~\citep{journals/pami/SegevHMCE17}, and applying model reuse to adapt different performance measures~\citep{journals/pami/LiTZ13}. Apart from algorithm design, theoretical foundations are recently established by stability~\citep{conf/icml/KuzborskijO13}, Rademacher complexity~\citep{journals/mlj/KuzborskijO17} and transformation functions~\citep{conf/nips/DuKSP17}.

Our paper proposes to handle concept drift problem via utilizing model reuse learning. The idea of leveraging previous knowledge is reminiscent of some previous work coping with concept drift by model reuse (transfer), like the temporal inductive transfer ($\mathtt{TIX}$) approach~\citep{conf/sigir/Forman06} and the diversity and transfer-based ensemble learning ($\mathtt{DTEL}$) approach~\citep{journals/tnnls/suny18}. Both of them are batch-style approaches, that is, they need to receive a batch of data each time, whereas ours can update either in an incremental style or a batch update mode. TIX concatenates the predictions from previous models into the feature of next data batch as the new data, and a new model is learned from the augmented data batch. DTEL chooses decision tree as the base learner, and builds a new tree by ``fine-tuning'' previous models by a direct tree structural adaptation. It maintains a fixed size model pool with the selection criteria based on diversity measurement. They both do not depict the reusability of previous models, which is carried out by $\mathtt{WeightUpdate}$ module in our approach. Last but not the least important, our approach is proposed with sound theoretical guarantees, in particular, we carry out a generalization justification on why and how to reuse previous models. Nevertheless, theirs are not theoretically clear in general.

\section{Proposed Approach}
\label{sec:proposed-approach}
In this section, we first illustrate the basic idea, and then identify two important modules in designing our proposed approach, i.e., $\mathtt{ModelUpdate}$ and $\mathtt{WeightUpdate}$. 

Specifically, we adopt a drift detection algorithm to split the whole data stream into epochs in which the distribution underlying is relatively smooth, and we only do the model update when detecting the concept drift or achieving the maximum update period $p$. As shown in Figure~\ref{figure:main-idea}, the drift detector $\mathfrak{D}$ will monitor the concept drift. When the drift is detected, instead of resetting the model pool and incrementally training a new model, we aim at leveraging the knowledge in previous models to enhance the overall performance by alleviating the cold start problem.

Basically, our approach consists of two important modules,\vspace{-2mm}
\begin{enumerate}[itemsep=0.5mm]
    \item[(1)] $\mathtt{ModelUpdate}$ by model reuse:  we leverage previous models to build the new model and update model pool, by making use of biased regularization multiple model reuse.
    \item[(2)] $\mathtt{WeightUpdate}$ by expert advice:  we associate each previous model with a weight representing the reusability towards current data. The weights are updated according to the performance of each model, in an exponential weighted average manner.
\end{enumerate}

\begin{figure}[!t]
\centering
\includegraphics[width=0.9\textwidth]{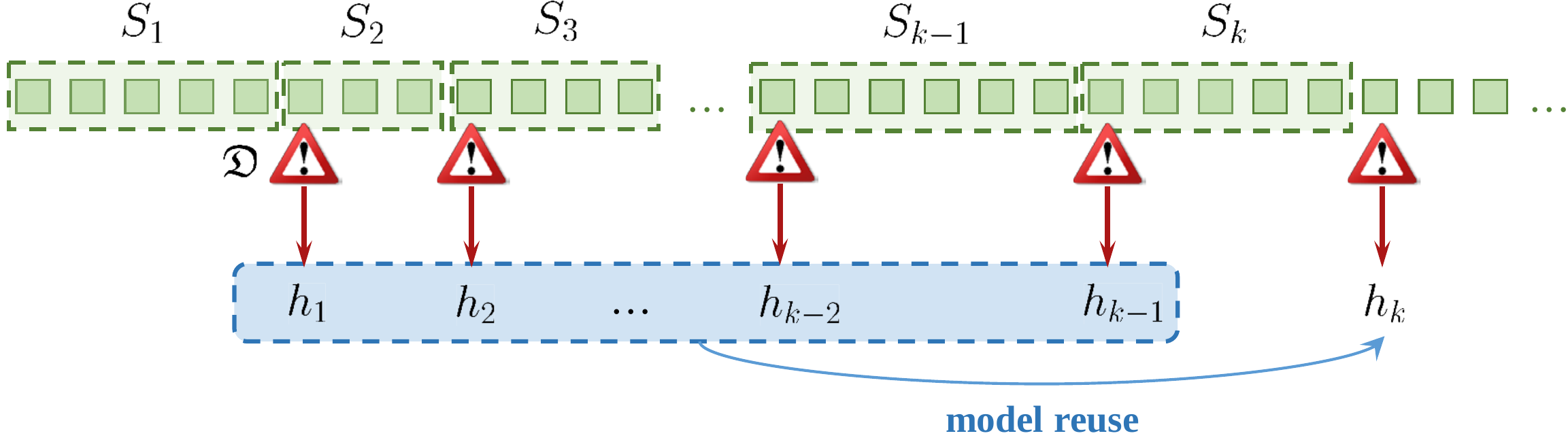}
\caption{Illustration of main idea: on one hand, we utilize the data items in current epoch $S_k$; on the other hand, we leverage the previous knowledge ($\{h_1,h_2,\ldots,h_{k-1}\}$) via model reuse.}
\label{figure:main-idea}
\end{figure}

\subsection{Model Update by Model Reuse}

We leverage previous models to adapt the current data epoch via model reuse by \emph{biased regularization}~\citep{conf/colt/ScholkopfHS01,journals/pami/TommasiOC14}. 

Consider the $k$-th model update as illustrated in Figure~\ref{figure:main-idea}, we desire to leverage previous models $\{h_1,\ldots,h_{k-1}\}$ and current data epoch $S_k$ to obtain a new model $h_k$. With a slight abuse of notation, we denote $S_k = \{(\mathbf{x}_1,y_1),\ldots,(\mathbf{x}_m,y_m)\}$. In this paper, we adopt linear classifier as the base model, and the model reuse by biased regularization can be formulated as
\begin{equation}
    \label{eq:by-biased-reg}
    \hat{\mathbf{w}}_k = \argmin_{\mathbf{w}} \left\lbrace\frac{1}{m}\sum_{i=1}^m \ell\left(\langle \mathbf{w} ,\mathbf{x}_i\rangle,y_i\right) + \mu \Omega(\mathbf{w} - \mathbf{w}_p)\right\rbrace,
\end{equation}
where $\ell:\mathcal{Y} \times \mathcal{Y} \rightarrow \mathbb{R}_+$ is the loss function, and $\Omega: \mathcal{H} \rightarrow \mathbb{R}_+$ is the regularizer. Besides, $\mu > 0$ is a positive trade-off regularization coefficient, and $\mathbf{w}_p$ is the linear weighted combination of previous models, namely, $\mathbf{w}_p = \sum_{j=1}^{k-1} \beta_j \hat{\mathbf{w}}_j$, where $\beta_j$ is the weight associated with previous model $h_j$, representing the reusability of each model on current data epoch. 

For simplicity, in this paper, we choose the square loss with $\ell_2$ regularization in practical implementation, essentially, Least Square Support Vector Machine (LS-SVM)~\citep{book/2002/suykens2002least}. It is shown~\citep{book/2002/suykens2002least} that the optimal solution can be expressed as $\mathbf{w} = \sum_{i=1}^m \alpha_i \mathbf{x}_i$, with $\boldsymbol{\alpha} = [\alpha_1,\ldots,\alpha_m]^\mathrm{T}$ solved by 
\begin{equation}
  \label{eq:linear-close-form}
  \begin{bmatrix}
    \mathbf{K}+\frac{1}{\mu}\mathbf{I} &  \mathbf{1} \\
    \mathbf{1} & 0
  \end{bmatrix}
  \begin{bmatrix}
    \boldsymbol{\alpha} \\
    b
  \end{bmatrix}
  = 
  \begin{bmatrix}
    \mathbf{y} - \sum_{j=1}^{k-1} \beta_j \hat{\mathbf{y}}_j \\
    0
  \end{bmatrix},
  \end{equation}
where $\mathbf{K}$ is the linear kernel matrix, i.e., $\mathbf{K}_{ij} = \mathbf{x}_i^T \mathbf{x}_j$. Besides, $\mathbf{y}$ and $\hat{\mathbf{y}}_j$ are the vectors containing labels of data stream and predictions of the previous $j$-th model, that is, $\mathbf{y}=[y_1,\ldots,y_m]^\mathrm{T}$ and $\hat{\mathbf{y}}_j = [\langle \hat{\mathbf{w}}_j,\mathbf{x}_1 \rangle, \ldots, \langle \hat{\mathbf{w}}_j,\mathbf{x}_m \rangle]^\mathrm{T}$. 

If the concept drift occurs very frequently or data stream accumulates for a long time, the size of model pool will explode supposing there is no delete operation. Thus, we set the maximum of model pool size as $K$. Apparently, we can keep $K$ of all models with largest diversity as done in~\citep{journals/tnnls/suny18}. For simplicity, we only keep the newest $K$ ones in the model pool.

\begin{myRemark}
The biased regularization model reuse learning~\eqref{eq:by-biased-reg} is not limited in binary scenario, and can be easily extended to multi-class scenario as,
\begin{equation}
	\label{eq:by-biased-reg-multi-class}
	\hat{W} = \argmin_{W} \left\{\frac{1}{m} \sum_{i=1}^m \ell \left( \rho_{h_{W,p}}(\mathbf{x}_i,y_i)\right) + \mu \Omega(W) \right\},
\end{equation}
where $h_{W,p}(\mathbf{x}) = W^\mathrm{T}\mathbf{x} + h_p(\mathbf{x})$, and $\rho_h(\mathbf{x},y)$ is the margin. We defer the notations and corresponding theoretical analyses in Section~\ref{sec:multi-class-model-reuse}. In addition, our approach is a framework, and can choose any multiple model reuse algorithm as the sub-routine. For instance, we can also choose model reuse by random forests~\citep{journals/pami/SegevHMCE17}.
\end{myRemark}

\begin{algorithm}[!t]
   \caption{\textsc{Condor}}
   \label{alg:Condor}
\begin{algorithmic}[1]
    \REQUIRE {Data stream $\{(\mathbf{x}_1,y_1),\ldots,(\mathbf{x}_T,y_T)\}$. Drift detector $\mathfrak{D}(\delta,\mathbf{x},y)$ with corresponding threshold $\delta$; step size $\eta$; maximum update period (epoch size) $p$; model pool size $K$.}
    \ENSURE Prediction $\hat{y}_{t}$, where $t= 1,\ldots,T$; and returned model pool $\mathcal{H}$.
    \STATE {Initialize model on first (or a couple of) data items: $h_1 \leftarrow \mathtt{Train}(\mathbf{x}_1)$, and $\mathcal{H}\leftarrow \{h_1\}$ ;}
    \STATE {Initialize weight $\beta_{1,1} \leftarrow 1$ ;}
    \FOR{$t=1$ {\bfseries to} $T$}
        \STATE{Receive $\mathbf{x}_t$};
        \FOR{$k = 1$ {\bfseries to} $\lvert \mathcal{H} \rvert$}
            \STATE{$\hat{y}_{t,k} \leftarrow h_k(\mathbf{x}_t)$;}
        \ENDFOR
        \STATE{$\hat{y}_t \leftarrow \sum_{k=1}^{\lvert \mathcal{H} \rvert} {\beta}_{t,k} \hat{y}_{t,k}/\sum_{k=1}^{\lvert \mathcal{H} \rvert} \beta_{t,k}$; }
        \STATE{Receive $y_t$};           
        \FOR{$k=1$ {\bfseries to} $\lvert \mathcal{H} \rvert$ }   \label{alg:weight-update-start}
            \STATE{$\beta_{t+1,k} \leftarrow \beta_{t,k} \exp\{-\eta \ell(\hat{y}_{t,k},y_t)\}$; // $\mathtt{WeightUpdate}$}    
        \ENDFOR       \label{alg:weight-update-end}
        \IF{$\mathfrak{D}(\delta,\mathbf{x}_t,y_t)>0$ or ($t \Mod p = 0$)}
            \STATE{$h \leftarrow \mathtt{ModelUpdate}(S_{\lvert \mathcal{H}\rvert},\mathcal{H},\{\beta_1,\ldots,\beta_{\lvert \mathcal{H}\rvert}\})$;} 
            \STATE{$\mathcal{H} \leftarrow \mathcal{H} \cup \{h\}$;} 
            \IF{$\lvert \mathcal{H} \rvert > K$}
                \STATE{Remove the oldest model from $\mathcal{H}$.}
            \ENDIF
            \FOR{$k=1$ {\bfseries to} $\lvert \mathcal{H} \rvert$}
                \STATE{Initialize the weights: ${\beta}_{1,k} \leftarrow 1/\lvert \mathcal{H}\rvert$};
            \ENDFOR
        \ENDIF
        \ENDFOR    
\end{algorithmic}
\end{algorithm}


\subsection{Weight Update by Expert Advice}
After $\mathtt{ModelUpdate}$ step, the weight distribution in the model pool $\mathcal{H}$ will reinitialize. We adopt a uniform initialization: $\beta_{1,k} = 1/\lvert \mathcal{H}\rvert$, for $k=1,\ldots,\lvert \mathcal{H}\rvert$.

After the initialization, we update weight of each model by expert advice~\citep{book/Cambridge/cesa2006prediction}. Specifically, when the new data item comes, we receive $\mathbf{x}_t$ and each previous model will provide its prediction $\hat{y}_{t,k}$, and the final prediction $\hat{y}_t$ is made based on the weighted combination of expert advice ($\hat{y}_{t,k}$s). Next, the true label is revealed as $y_t$, and we will update the weights according to the loss each model suffers, in an exponential weighted manner,
\[
\beta_{t+1,k} \leftarrow \beta_{t,k} \exp\{-\eta \ell(\hat{y}_{t,k},y_t)\}.
\] 

The overall procedure of proposed approach \textsc{Condor} is summarized in Algorithm~\ref{alg:Condor}.

\section{Theoretical Analysis}
\label{sec:theoretical-analysis}
In this section, we provide theoretical analysis both locally and globally. \vspace{-2mm}
\begin{enumerate}[itemsep=0.5mm]
	\item[(1)] Local analysis: consider both generalization and regret aspects on each epoch locally;
	\item[(2)] Global analysis: examine regret on the whole data stream globally.
\end{enumerate}\vspace{-2mm}

Besides, in local analysis, we also provide the multi-class model reuse analysis, and we let it an independent subsection to better present the results.

\subsection{Local Analysis}
The local analysis means that we scrutinize the performance on a particular epoch. On one hand, we are concerned about the generalization ability of the model obtained by $\mathtt{ModelUpdate}$ module. Second, we study the quality of learned weights by $\mathtt{WeightUpdate}$ module and the cumulative regret of prediction. 

Let us consider the epoch $S_k$, the $\mathtt{ModelUpdate}$ module reuses previous models $h_1,\ldots,h_{k-1}$ to help built new model $h_k$, as shown in Figure~\ref{figure:main-idea}. To simplify the presentation, we introduce some notations. Suppose the length of data stream $S$ is $T$, and is partitioned into $k$ epochs, $S_2,\ldots,S_k$.\footnote{Here, we start from epoch $2$, since the first epoch cannot utilize $\mathtt{WeightUpdate}$.} For epoch $S_k$, we assume the distribution is identical, i.e., $S_k$ is a sample of $m_k$ points drawn i.i.d. according to distribution $\mathcal{D}_k$, where $m_k$ denotes its length. 

First, we conduct generalization analysis on $\mathtt{ModelUpdate}$ module. Define the risk and empirical risk of hypothesis (model) $h$ on epoch $S_k$ by
\begin{equation*}
	R(h) = \mathbb{E}_{(\mathbf{x},y)\sim \mathcal{D}_k} \left[\ell(h(\mathbf{x}),y)\right], \ \  \hat{R}(h) = \frac{1}{m_k}\sum_{i\in S_k} \ell(h(\mathbf{x}_i),y_i).
\end{equation*}
Here, with a slight abuse of notations, we also adopt $S_k$ to denote the index included in the epoch, and $\hat{R}(h)$ instead of $\hat{R}_{S_k}(h)$ for simplicity. The new model $h_k$ is built and updated on epoch $S_k$ via $\mathtt{ModelUpdate}$ module, then we have the following generalization error bound.

\begin{myThm}
\label{thm:generalization-main-theorem}
Assume that the non-negative loss function $\ell:\mathcal{Y}\times \mathcal{Y} \rightarrow \mathbb{R}_+$ is bounded by $M\geq 0$, and is $L$-Lipschitz continuous. Also, assume the regularizer $\Omega: \mathcal{H} \rightarrow \mathbb{R}_+$ is a non-negative and $\lambda$-strongly convex function w.r.t. a norm $\lVert \cdot \rVert$. Given the model pool $\{h_1,h_2,\ldots,h_{k-1}\}$ with $\mathbf{h}_{p}(\mathbf{x}):= [h_1(\mathbf{x}),h_2(\mathbf{x}),\ldots,h_{k-1}(\mathbf{x})]^\mathrm{T}$, and denote $h_p$ be a linear combination of previous models, i.e., $h_p(\mathbf{x}) = \langle \boldsymbol{\beta},\mathbf{h}_p(\mathbf{x})\rangle$ and supposing $\Omega(\boldsymbol{\beta}) \leq \rho$ and $\rho = O(1/m)$. Let $h_k$ be the model returned by $\mathtt{ModelUpdate}$. Then, for any $\delta > 0$, with probability at least $1-\delta$, the following holds,\footnote{We use $m$ instead of $m_k$ for simplicity.}
\begin{equation*}
	\label{eq:main-results}
		R(h_k) - \hat{R}(h_k) \leq 4L\sqrt{\frac{\epsilon_1}{m}} + 3\sqrt{\frac{\epsilon_2\log(1/\delta)}{m}}+ \frac{3M\log(1/\delta)}{4m},	
\end{equation*}
where $\epsilon_1 = \frac{B^2 R_p}{\lambda \mu}+ \frac{C^2\rho}{\lambda}$ and $\epsilon_2 = \frac{M}{4} R_p + 4LM\sqrt{\frac{B^2 R_p + C^2 \mu \rho}{\lambda \mu m}}$. Besides, $B = \sup_{\mathbf{x}\in \mathcal{X}} \lVert \mathbf{x}\rVert_\star$ and $C = \sup_{\mathbf{x}\in \mathcal{X}} \lVert \mathbf{h}_p(\mathbf{x})\rVert_\star$.

To better present the results, we only keep the leading term w.r.t. $m$ and $R_p$, and we have 
\begin{equation}
	\label{eq:main-results-order}
		R(h_k) - \hat{R}(h_k) = O\left( \frac{1}{\sqrt{m}}\left( \sqrt{R_p} + \sqrt{\frac{R_p}{\lambda \mu}} +\sqrt[4]{\frac{R_p}{\lambda \mu m}}\right) + \frac{1}{m}\left(\sqrt{\frac{1}{\lambda}} + \sqrt[4]{\frac{1}{\lambda}}\right)\right),	
\end{equation}
where $R_p = R(h_p) = \mathbb{E}_{S_k}[\ell(h_p)]$, representing the risk of reusing model on current distribution.
\end{myThm}

To better present our main result, we defer the proof of Theorem~\ref{thm:generalization-main-theorem} in Appendix~\ref{sec:proof-thm1}.

\begin{myRemark}
 Eq.~\eqref{eq:main-results-order} shows that $\mathtt{ModelUpdate}$ procedure enjoys an $O(1/\sqrt{m})$ generalization bound under certain conditions. In particular, when the previous models are sufficiently good, that is, have a small risk on the current distribution $\mathcal{D}_k$ (i.e., when $R_p\rightarrow 0$), we can obtain an $O(1/m)$ bound, a fast rate convergence guarantee. This implies the effectiveness of leveraging and reusing previous models to current data, especially if we can reuse previous models properly (as we will illustrate in the following paragraph).
\end{myRemark}

\begin{myRemark}
The main techniques in the proof are inspired by~\cite{journals/mlj/KuzborskijO17}, but we differ in two aspects. First, we only assume the Lipschitz condition, and thus their results are not suitable under our conditions. Second, we extended the analysis of model reuse to multi-class scenarios, and include the results in Appendix~\ref{sec:multi-class-model-reuse} for a better presentation.
\end{myRemark}

Before the next presentation, we need to introduce more notations. Let $L_T$ be as the global cumulative loss on the whole data stream $S$. on epoch $S_k$, let $L_{S_k}$ be as the local cumulative loss suffered by our approach, and $L_{S_k}^{(j)}$ as the local cumulative loss suffered by the previous model $h_j$,
\begin{equation}
	\label{eq:cumulative-loss}
	L_T = \sum_{i=1}^T \ell\left(\hat{y}_i,y_i\right), \quad L_{S_k} = \sum_{i\in S_k} \ell\left(\hat{y}_i,y_i\right), \quad L_{S_k}^{(j)} = \sum_{i\in S_k} \ell\left(h_j(\mathbf{x}_i),y_i\right).
\end{equation}

Next, we show $\mathtt{WeightUpdate}$ returns a good weight distribution, implying our approach can reuse previous models properly. In fact, we have the following observation regarding to weight distribution.
\begin{myObservation}[Weight Concentration]
\label{obser:weight-concentration}
\emph{During the $\mathtt{WeightUpdate}$ procedure in epoch $S_k$, the weights will concentrate on those previous models who suffer a small cumulative loss on $S_k$.} 
\end{myObservation}
\begin{proof}
By a simple analysis on the $\mathtt{WeightUpdate}$ procedure, we know that the weight associated with the $j$-th previous model is equal to $\beta_{1,j}\exp\{-\eta L_{S_k}^{(j)}\}$, where $j =1,\ldots,k-1$.
\end{proof}
Though the observation seems quite straightforward, it plays an important role in making our approach successful. The statement guarantees that the algorithm adaptively assigns more weights on better-fit previous models, which essentially depicts the `reusability' of each model. We also conduct additional experiments to support this point in Appendix~\ref{sec:weight-concentration}.


Third, we show that our approach can benefit from \emph{recurring concept drift} scenarios. Here, we adopt the concept of \emph{cumulative regret} (or \emph{regret}) from online learning~\citep{conf/icml/Zinkevich03} as the performance measurement.

\begin{myThm}[Improved Local Regret~\citep{book/Cambridge/cesa2006prediction}]
\label{thm:regret-local-improved}
Assume that the loss function $\ell:\mathcal{Y}\times \mathcal{Y} \rightarrow \mathbb{R}_+$ is convex in its first argument and takes the values in $[0,1]$. Besides, the step size is set as $\eta = \ln(1+\sqrt{2\ln (k-1)/L_{j_k^*}})$, where $L_{j^*_k} = \min_{j=1,\ldots,k-1} L_{S_k}^{(j)}$ is the cumulative loss of the best-fit previous model and is supposed to be known in advance. Then, we have,
 \[
	\mathrm{Regret}_{S_k} = L_{S_k} - L_{j_k^*} \leq \sqrt{2L_{j_k^*} \ln (k-1)} + \ln (k-1).
\] 
\end{myThm}

\begin{proof}
Refer to the proof presented in page 21 in Chapter 2.4 of~\cite{book/Cambridge/cesa2006prediction}.
\end{proof}

Above statement shows that the order of regret bound can be substantially improved from a typical $O(m_k)$ to $O\left(\ln k\right)$,  independent from the number of data items in the epoch, providing $L_{k^*} \ll \sqrt{m_k}$, that is, the cumulative loss of the best-fit previous model is small.

\begin{myRemark}
Theorem~\ref{thm:regret-local-improved} implies that if the concept of epoch $S_k$ or a similar concept has emerged previously, our approach enjoys a substantially improved local regret providing a proper step size is chosen. This accords to our intuition on why model reuse helps for concept drift data stream. In many situations, although the distribution underlying might change over time, the concepts can be recurring, i.e., disappear and re-appear~\citep{journals/kais/KatakisTV10,journals/kais/GamaK14}. Thus, the statement shows that our approach can benefit from such recurring concepts, and we empirically support this point in Appendix~\ref{sec:recurring}.
\end{myRemark}

\subsection{Global Analysis}
The global analysis means that we study the overall performance on the whole data stream. We provide the global dynamic regret as follows.
\begin{myThm}[Global Dynamic Regret]
\label{thm:dynamic-regret}
Assume that the loss function $\ell:\mathcal{Y}\times \mathcal{Y} \rightarrow \mathbb{R}_+$ is convex in its first argument and takes the values in $[0,1]$. Assume that the step size in $S_k$ epoch is set as $\eta_k = \sqrt{(8\ln(k-1))/m_k}$,\footnote{The choice of $\eta$ requires the knowledge of epoch size, which can be eliminated by \emph{doubling trick}, at the price of a small constant factor~\citep{journals/jacm/Cesa-BianchiFHHSW97}.} then we have\vspace{-2mm}
\begin{equation}
	\label{eq:dynamic-regret}
	\mathrm{Regret}_T = L_T - \sum_{k=2}^K L_{j_k^*} \leq  \sqrt{\left(\sum_{k=1}^{K-1}\ln k\right)T/2},\vspace{-2mm}
\end{equation}
where $j^*_k = \argmin_{j=1,\ldots,k-1} L_{S_k}^{(j)}$.
\end{myThm}
The proof of global dynamic regret is built on the local static regret analysis in each epoch. And we can see that for data stream with a fix length $T$, the more concept drifts occur (i.e., larger $K$), the larger the regret will be. This accords with our intuition, on one hand, the sum of best-fit local cumulative loss ($\sum_{k=2}^K L_{j_k^*}$) is going to be compressed with more previous models. On the other hand, the learning problem becomes definitely harder as concept drift occurs more frequently.

Our $\mathtt{WeightUpdate}$ strategy is essentially exponentially weighted average forecaster~\citep{book/Cambridge/cesa2006prediction}, and thus we have the following local regret guarantee in each epoch. 
\begin{myLemma}[Theorem 2.2 in~\cite{book/Cambridge/cesa2006prediction}]
\label{lemma:local-regret-convex}
Assume that the loss function $\ell:\mathcal{Y}\times \mathcal{Y} \rightarrow \mathbb{R}_+$ is convex in its first argument and takes the values in $[0,1]$. Assume that the step size is set as $\eta = \sqrt{8\ln (k-1)/m_k}$, then we have 
\[
	L_{S_k} \leq \min_{j=1,\ldots,k-1} L_{S_k}^{(j)} + \sqrt{(m_k/2) \ln (k-1)}.
\]
\end{myLemma}
\begin{proof}
The proof is based on a simple reduction from our scenario to standard \emph{exponentially weighted average forecaster}. For epoch $S_k$, let previous models pool $\{h_1,h_2,\ldots,h_{k-1}\}$ be as the expert pool. Then, plugging the expert number $N = k-1$ and number of instances $n=m_k$ into Theorem 2.2 in~\cite{book/Cambridge/cesa2006prediction}, we obtain the statement. 

Besides, the proof of exponentially weighted average forecaster is standard, which utilizes \textit{potential function method}~\citep{book/Cambridge/cesa2006prediction,book/MIT/mohri2012foundations}. For a detailed proof, one can refer to the proof presented in page 157-159 in Chapter 7 of book~\citep{book/MIT/mohri2012foundations}.
\end{proof}

Now, we proceed to prove Theorem~\ref{thm:dynamic-regret}.
\begin{proof}
Our proof relies on the application of local static regret analysis. Since $S_2,\ldots,S_K$ is a partition of the whole period $T$, we apply Lemma~\ref{lemma:local-regret-convex} on each epoch locally and obtain
\begin{equation}
	\label{eq:local-regret-convex}
	L_{S_k} \leq \min_{j=1,\ldots,k-1} L_{S_k}^{(j)} + \sqrt{(m_k/2) \ln (k-1)}.	
\end{equation}

Sum over the index of $j$ from $1$ to $k-1$, we have 
\begin{eqnarray}\nonumber
	L_T & = & \sum_{k=2}^K L_{S_k} \\
	\label{eq:global-1}
	&\leq & \sum_{k=2}^K L_{j_k^{*}} + \sum_{k=2}^K \sqrt{(m_k/2) \ln (k-1)}\\ 
	\label{eq:global-2}
	& \leq & \sum_{k=2}^K L_{j_k^{*}} + \sqrt{\sum_{k=2}^K (m_k/2) \sum_{k=2}^K \ln (k-1)}\\
	& =& \sum_{k=2}^K L_{j_k^*}+\sqrt{\left(\sum_{k=1}^{K-1}\ln k\right)T/2}\nonumber
\end{eqnarray}
where \eqref{eq:global-1} holds by substituting \eqref{eq:local-regret-convex} into each epoch $S_k$, and \eqref{eq:global-2} holds by applying Cauchy-Schwartz inequality.
\end{proof}

\begin{myRemark}
Essentially, the regret bound in \eqref{eq:dynamic-regret} is different from the traditional (static) regret bound. It measures the difference between the global cumulative loss with the sum of local cumulative loss suffered by previous best-fit models. Namely, our competitor changes in each epoch, which depicts the distribution change in the sequence, and thus is more suitable to be the performance measurement in non-stationary environments.
\end{myRemark}

\subsection{Multi-Class Model Reuse Learning}
In multi-class learning scenarios, the notations are slightly different from those in binary case. We first introduce the new notations for a clear presentation. 

Let $\mathcal{X}$ denote the input feature space and $\mathcal{Y} = \{1,2,\ldots,c\}$ denote the target label space. Our analysis acts on the last data epoch $S_k = \{(\mathbf{x}_1,{y}_1),\ldots,(\mathbf{x}_{m_k},{y}_{m_k})\}$, a sample of $m_k$ points drawn i.i.d. according to distribution $\mathcal{D}_k$, where $\mathbf{x}_i\in \mathbb{R}^d$ and ${y}_i \in \mathcal{Y}$ with only a single class from $\{1,\ldots,c\}$. Given the multi-class hypothesis set $\mathcal{H}$, any hypothesis $h\in \mathcal{H}$ maps from $\mathcal{X}\times \mathcal{Y}\rightarrow \mathbb{R}$, and makes the prediction by $\mathbf{x}\rightarrow \argmax_{{y}\in \mathcal{Y}} h(\mathbf{x},{y})$. This naturally rises the definition of \emph{margin} $\rho_h(\mathbf{x},{y})$ of the hypothesis $h$ at a labeled instance $(\mathbf{x},{y})$,
\[
	\rho_h(\mathbf{x},{y}) = h(\mathbf{x},{y}) - \max_{y' \neq y} h(\mathbf{x},y').
\]

The non-negative loss function $\ell:\mathbb{R} \rightarrow \mathbb{R}$ is bounded by $M>0$. Besides, we assume the loss function is \emph{regular loss} defined in~\cite{conf/nips/LeiDBK15}. 

\begin{myDef}[Regular Loss] 
	\label{def:regular-loss}
	We call a loss function $\ell:\mathbb{R} \rightarrow \mathbb{R}$ is a $L$-regular if it satisfies the following properties (Cf. Definition 2 in~\cite{conf/nips/LeiDBK15}):
	\begin{enumerate}[label=(\roman*)]
		\item $\ell(t)$ bounds the $0$-$1$ loss from above: $\ell(t) \geq \mathbf{1}_{t\leq 0}$; 
		\item $\ell(t)$ is $L$-Lipschitz continuous, i.e., $\lvert \ell(t_1) - \ell(t_2)\rvert \leq L \lvert t_1 - t_2\rvert$;
		\item $\ell(t)$ is decreasing and it has a zero point $c_{\ell}$, i.e., there exists a $c_{\ell}$ such that $\ell(c_{\ell}) = 0$. \label{def:regular-loss-property-3}
	\end{enumerate}
\end{myDef}

Then the risk and empirical risk of a hypothesis $h$ on epoch $S_k$ are defined by
\begin{equation*}
	R(h) = \mathbb{E}_{(\mathbf{x},y)\sim \mathcal{D}_k} \left[\mathbf{1}_{\rho_h(\mathbf{x},y)} \leq 0\right], \ \  \hat{R}_S(h) = \frac{1}{m_k}\sum_{i\in S_k} \ell(\rho_h(\mathbf{x}_i,y_i)).
\end{equation*}

Our goal is to provide a generalization analysis, namely, to prove that risk $R(h)$ approaches empirical risk $\hat{R}(h)$ as number of instances $m_k$ increases, and establish the convergence rate. Since $\mathbb{E}_{(\mathbf{x},y)\sim \mathcal{D}_k}[\hat{R}(h)] \neq R(h)$, thus, we cannot directly utilize concentration inequalities to help analysis. To make this simpler, we need to introduce the risk w.r.t. loss function $\ell$, 
\[
	R_{\ell}(h) = \mathbb{E}_{(\mathbf{x},y)\sim \mathcal{D}_k} \left[\ell(\rho_h(\mathbf{x},y))\right].
\] 

From property~\ref{def:regular-loss-property-3} in Definition~\ref{def:regular-loss}, we know that the risk $R(h)$ is a lower bound on $R_{\ell}(h)$, that is $R(h)\leq R_{\ell}(h)$. Thus, we only need to establish generalization bound between $R_{\ell}(h)$ and $\hat{R}(h)$. Apparently, $\mathbb{E}_{(\mathbf{x},y)\sim \mathcal{D}_k}[\hat{R}(h)] =  R_{\ell}(h)$, thus we can utilize concentration inequalities again.

First, we identify the optimization formulation of multi-class biased regularization model reuse, 
\begin{equation}
	\label{eq:biased-regularization-multi-class}
	\hat{W} = \argmin_{W} \left\{\frac{1}{m} \sum_{i=1}^m \ell \left( \rho_{h_{W,p}}(\mathbf{x}_i,y_i)\right) + \lambda \Omega(W) \right\},
\end{equation}
where $h_{W,p}(\mathbf{x}) = W^\textrm{T}\mathbf{x} + h_p(\mathbf{x})$.

We specify the regularizer as square of Frobenius norm, namely, $\Omega(W) = \lVert W \rVert_F^2$, and provide the following generalization error bound.
\begin{myThm}
\label{thm:generalization-main-theorem-mc}
Let $H\subseteq \mathbb{R}^{\mathcal{X}\times \mathcal{Y}}$ be a hypothesis set with $\mathcal{Y} = \{1,2,\ldots,c\}$. Assume that the non-negative loss function $\ell:\mathbb{R} \rightarrow \mathbb{R}_+$ is $L$-regular. Given the model pool $\{h_1,h_2,\ldots,h_{k-1}\}$ with $\mathbf{h}_{p}(\mathbf{x}):= [h_1(\mathbf{x}),h_2(\mathbf{x}),\ldots,h_{k-1}(\mathbf{x})]^\mathrm{T}$, and denote $h_p$ be a linear combination of previous models, i.e., $h_p(\mathbf{x}) = \langle \boldsymbol{\beta},\mathbf{h}_p(\mathbf{x})\rangle$ and supposing $\lVert \boldsymbol{\beta} \rVert^2 \leq 2\rho$ and $\rho = O(1/m)$. Let $h_k$ be the model returned by $\mathtt{ModelUpdate}$. Then, for any $\delta > 0$, with probability at least $1-\delta$, the following holds,\footnote{We use $m$ instead of $m_k$ for simplicity.}
\begin{equation*}
	\label{eq:main-results-mc}
	\begin{split}
	R(h_{\hat{W},p}) - \hat{R}_S(h_{\hat{W},p}) &\leq 2L c^2 \sqrt{\frac{\epsilon_1}{ m}} + 3\sqrt{\frac{\epsilon_2\log(1/\delta)}{4m}}+ \frac{3M\log(1/\delta)}{4m}.	
	\end{split}
\end{equation*}

where $\epsilon_1 = \frac{2B^2L R_p}{\lambda}+ C^2\rho$ and $\epsilon_2 = M\left(8L c^2 \sqrt{\frac{2(B^2 R_p + C^2\lambda\rho)}{\lambda m}}+R_p\right)$. Besides, $B = \sup_{\mathbf{x}\in \mathcal{X}} \lVert \mathbf{x}\rVert$ and $C = \sup_{\mathbf{x}\in \mathcal{X}} \lVert \mathbf{h}_p(\mathbf{x})\rVert$.

To better present the results, we only keep the leading term w.r.t. $m$ and $R_p$, and we have 
\begin{equation}
	\label{eq:main-results-order-mc}
		R(h_k) - \hat{R}(h_k) = O\left( \frac{c^2}{\sqrt{m}}\left( \sqrt{R_p} + \sqrt{\frac{L R_p}{\lambda}} +\sqrt[4]{\frac{L^2 R_p}{\lambda m}}\right) + \frac{1}{m}\sqrt[4]{\frac{1}{\lambda}}\right),	
\end{equation}
where $R_p = R(h_p) = \mathbb{E}_{S_k}[\ell(h_p)]$, representing the risk of reusing model on current distribution.
\end{myThm}

\begin{myRemark}
From Theorem~\ref{thm:generalization-main-theorem-mc}, we can see that the main result and conclusion in multi-class case is very similar to that in binary case. In \eqref{eq:main-results-order-mc}, we can see that \textsc{Condor} enjoys an $O(1/\sqrt{m})$ order generalization bound, which is consistent to the common learning guarantees. More importantly, \textsc{Condor} enjoys an $O(1/m)$ order fast rate generalization guarantees, when $R_p\rightarrow 0$, namely, when the previous model $h_p$ is highly `reusable' with respect to the current data. This shows the effectiveness of leveraging and reusing previous models to help build new model in \textsc{Condor}, in multi-class scenarios.
\end{myRemark}
\section{Experiments}
\label{sec:experiments}
In this section, we examine the effectiveness of \textsc{Condor} on both synthetic and real-world concept drift datasets. Additional experimental regarding to weight concentration, recurring concept drift, parameter study and robustness comparisons are presented in Appendix~\ref{sec:more-experiments}.

\textbf{Compared Approaches.} We conduct the comparisons with two classes of state-of-the-art concept drift approaches. The first class is the \emph{ensemble category}, including (a) $\mathtt{Learn}^{++}.\mathtt{NSE}$~\citep{journals/tnn/ElwellP11}, (b) $\mathtt{DWM}$~\citep{conf/icdm/KolterM03,journals/jmlr/KolterM07} and (c) $\mathtt{AddExp}$~\citep{conf/icml/KolterM05}. The second class is the \emph{model-reuse category}, including (d) $\mathtt{DTEL}$~\citep{journals/tnnls/suny18} and (e) $\mathtt{TIX}$~\citep{conf/sigir/Forman06}. Essentially, DTEL and TIX also adopt ensemble idea, we classify them into model-reuse category just to highlight their model reuse strategies. 

\textbf{Settings.} In our experiments, we choose ADWIN algorithm~\citep{conf/sdm/BifetG07} as the the drift detector $\mathfrak{D}$ with default parameter setting reported in their paper and source code. Besides, for all the approaches, we set the maximum update period (epoch size) $p = 50$,\footnote{Except for Covertype and GasSensor datasets, we set $p=200$, since Covertype is extremely large with 581,012 data items in total and GasSensor is the multi-class dataset which has a higher sample complexity.} and model pool size $K=25$.

\begin{table}[!h]
\scriptsize 
\centering
\caption{Basic statistics of datasets with concept drift.}
\label{table:dataset-info}
\resizebox{0.95\textwidth}{!}{
\begin{tabular}{lrcc|lrrc}\toprule
\multicolumn{1}{c}{Dataset}    & \multicolumn{1}{c}{\# instance} & \multicolumn{1}{c}{\# dim}  & \multicolumn{1}{c|}{\# class} & \multicolumn{1}{c}{Dataset}    & \multicolumn{1}{c}{\# instance} & \multicolumn{1}{c}{\# dim} & \multicolumn{1}{c}{\# class}\\ \midrule
SEA200A      & 24,000     & 3  & 2 & GEARS-2C-2D  & 200,000     & 2     & 2 \\
SEA200G      & 24,000     & 3  & 2 & Usenet-1     & 1,500       & 100 	& 2 \\
SEA500G      & 60,000     & 3  & 2 & Usenet-2     & 1,500       & 100 	& 2 \\
CIR500G      & 60,000     & 3  & 2 & Luxembourg   & 1,900       & 32  	& 2 \\
SINE500G     & 60,000     & 2  & 2 & Spam         & 9,324       & 500 	& 2 \\
STA500G      & 60,000     & 3  & 2 & Email        & 1,500       & 913 	& 2 \\
1CDT      	 & 16,000     & 2  & 2 & Weather      & 18,159      & 8   	& 2 \\
1CHT      	 & 16,000     & 2  & 2 & GasSensor    & 4,450       & 129   & 6 \\
UG-2C-2D     & 100,000    & 2  & 2 & Powersupply  & 29,928      & 2   	& 2 \\
UG-2C-3D     & 200,000    & 3  & 2 & Electricity  & 45,312      & 8   	& 2 \\
UG-2C-5D     & 200,000    & 5  & 2 & Covertype    & 581,012     & 54    & 2 \\ \bottomrule
\end{tabular}}
\end{table}

\textbf{Synthetic Datasets.} As it is not realistic to foreknow the detailed concept drift information of real-world datasets, like the start, the end of change and so on. We employ six widely used synthetic datasets SEA, CIR, SIN and STA with corresponding variants into experiments. Besides, another six synthetic datasets for binary classification are also adopted: 1CDT, 1CHT, UG-2C-2D, UG-2C-3D, UG-2C-5D, and GEARS-2C-2D. A brief statistics are summarized in Table~\ref{table:dataset-info}. We provide the datasets information in Appendix~\ref{sec:dataset-description}.

We plot the \emph{holdout accuracy} comparisons over three synthetic datasets, SEA200A, SEA200G and SEA500G. Since some of compared approaches are batch style, following the splitting setting in \cite{journals/tnnls/suny18}, we split them into 120 epochs to have a clear presentation. The holdout accuracy is calculated over testing data generated according to the identical distribution as training data at each time stamp. For SEA and its variants, the distribution changes for seven times. From Figure~\ref{fig:synthetic-holdout}, we can see that all the approaches drop when an abrupt concept drift occurs. Nevertheless, our approach \textsc{Condor} is relatively stable and rises up rapidly with more data items coming, with the highest accuracy compared with other approaches, which validates its effectiveness.

\begin{figure}[!t]
    \centering
    \subfigure[SEA200A]{ \label{fig:drift-SEA200A} 
        \includegraphics[clip, trim=3.6cm 9.2cm 4.2cm 10.0cm,width=0.3\textwidth]{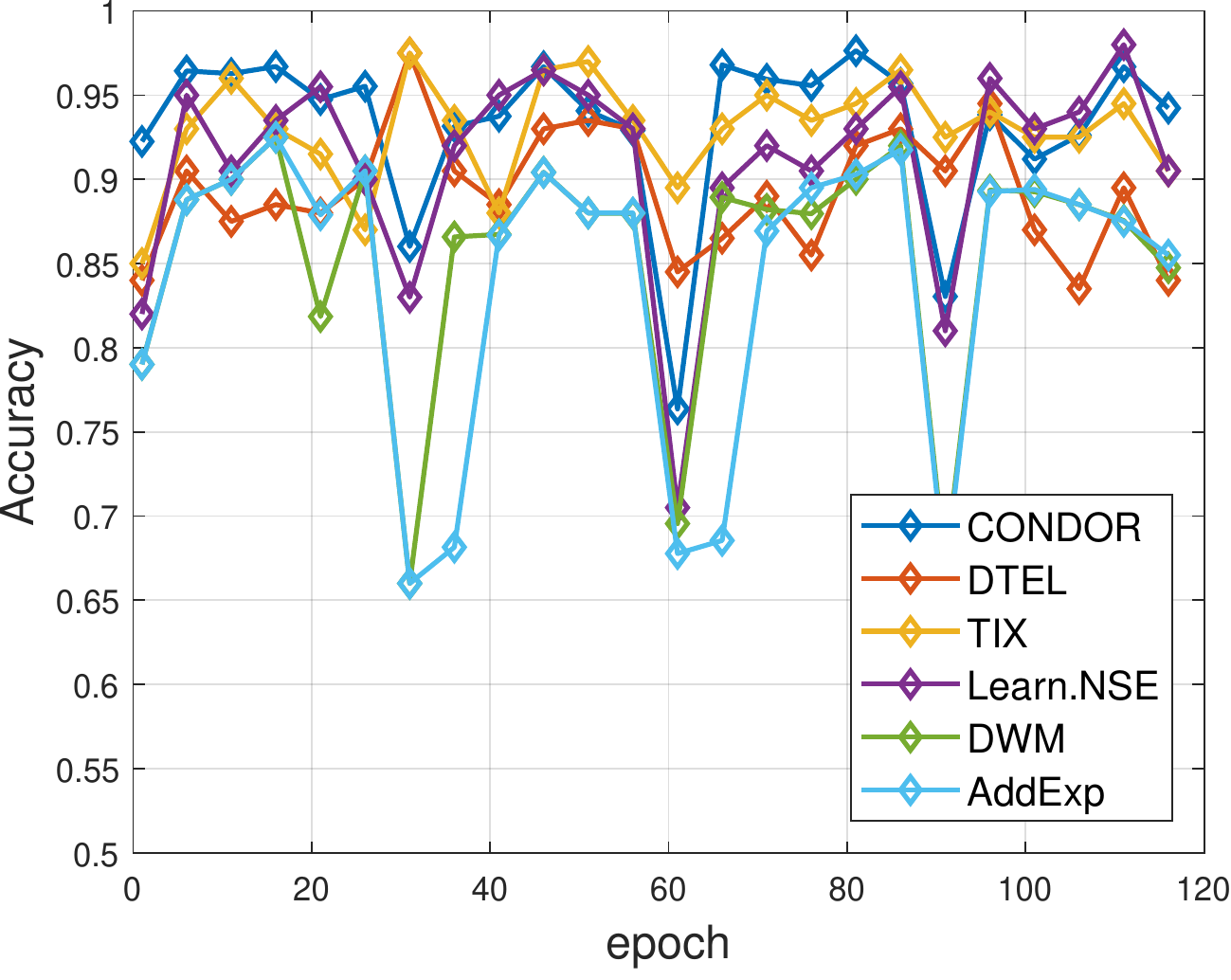}}\hspace{3mm}
    \subfigure[SEA200G]{ \label{fig:drift-SEA200G}
        \includegraphics[clip, trim=3.6cm 9.2cm 4.2cm 10.0cm, width=0.3\textwidth]{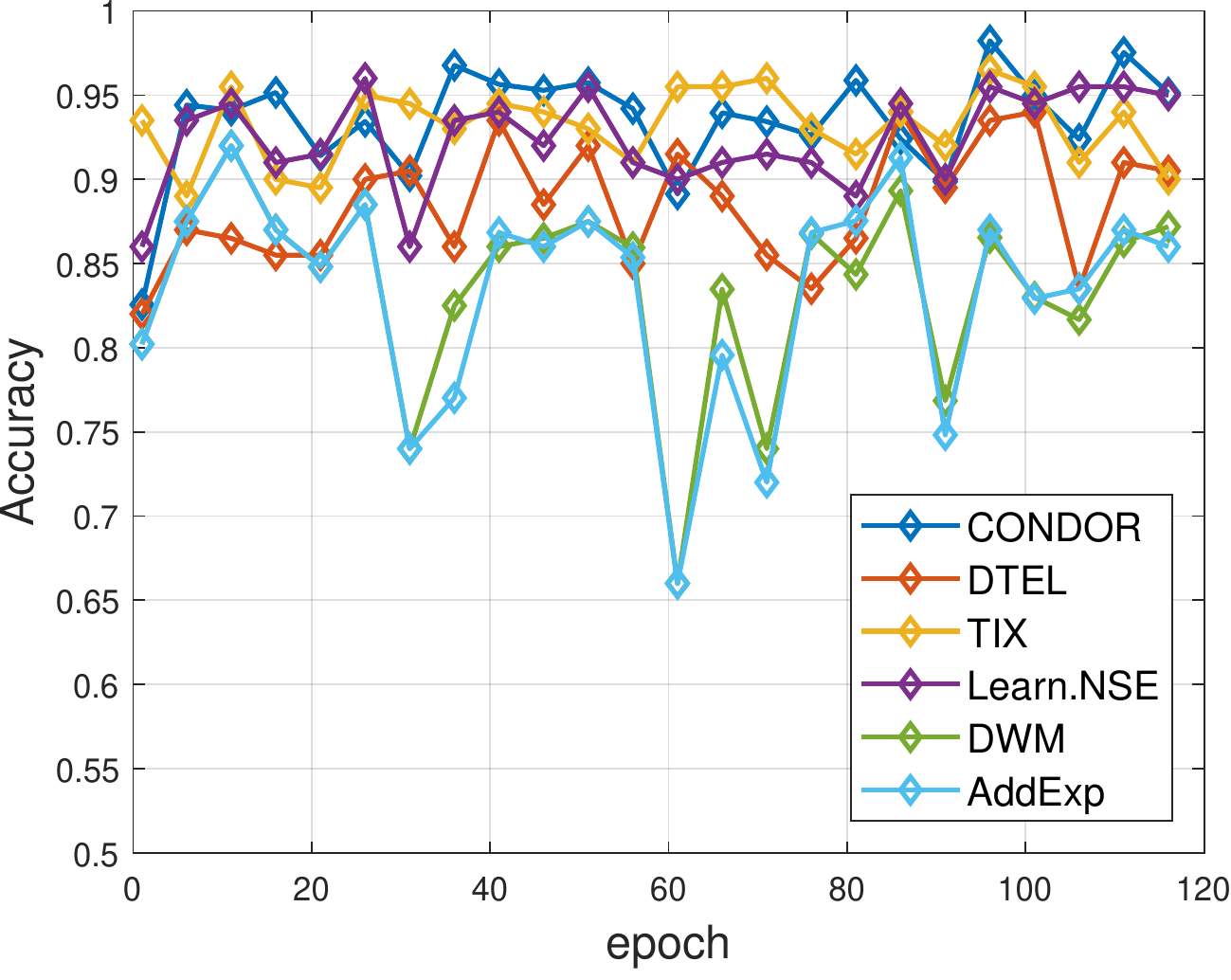}} \hspace{3mm}   
    \subfigure[SEA500G]{ \label{fig:drift-SEA500G}
        \includegraphics[clip, trim=3.6cm 9.2cm 4.2cm 10.0cm, width=0.3\textwidth]{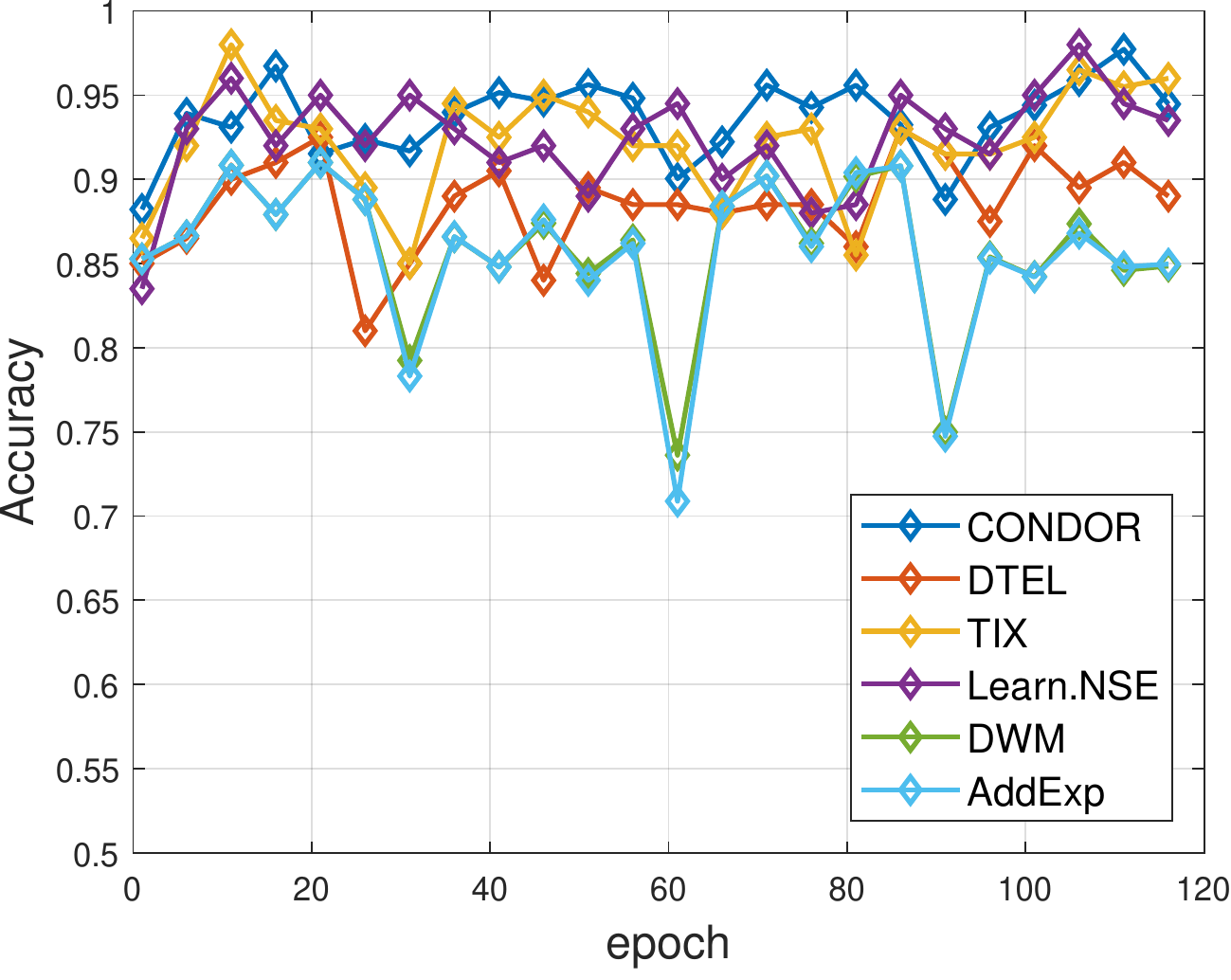}}
    \caption{Holdout accuracy comparisons on three synthetic datasets.}
    \label{fig:synthetic-holdout}
\end{figure}

\begin{table}[!t]
\centering
\scriptsize
\caption{\small{Performance comparisons on synthetic and real-world datasets. Besides, $\bullet$ ($\circ$) indicates our approach \textsc{Condor} is significantly better (worse) than compared approaches (paired $t$-tests at 95\% significance level).} }
\label{table:accuracy-all}
\resizebox{\textwidth}{!}{
\begin{tabular}{l|lllll|c}\toprule
\multirow{2}{*}{Dataset} &  \multicolumn{3}{c}{\textit{Ensemble Category}} & \multicolumn{2}{c|}{\textit{Model-Reuse Category}} & \multicolumn{1}{c}{\textit{Ours}} \\
                           & \multicolumn{1}{c}{$\mathtt{Learn}^{++}.\mathtt{NSE}$}  & \multicolumn{1}{c}{$\mathtt{DWM}$}  & \multicolumn{1}{c}{$\mathtt{AddExp}$}   & \multicolumn{1}{c}{$\mathtt{DTEL}$} & \multicolumn{1}{c|}{$\mathtt{TIX}$} & \multicolumn{1}{c}{\textsc{Condor}}\\ \midrule
SEA200A      & 84.48 $\pm$ 0.19 $\bullet$ & 86.07 $\pm$ 0.30 $\bullet$ & 84.35 $\pm$ 0.86 $\bullet$   & 80.50 $\pm$ 0.58 $\bullet$ 		& 82.79  $\pm$ 0.27 $\bullet$	& \textbf{86.67 $\pm$ 0.21} \\
SEA200G      & 85.48 $\pm$ 0.33 $\bullet$ & 86.92 $\pm$ 0.13 $\bullet$ & 85.54 $\pm$ 0.69 $\bullet$   & 80.73 $\pm$ 0.19 $\bullet$ 		& 82.95  $\pm$ 0.12 $\bullet$	& \textbf{87.63 $\pm$ 0.24} \\
SEA500G      & 86.03 $\pm$ 0.19 $\bullet$ & 87.63 $\pm$ 0.06 $\bullet$ & 87.14 $\pm$ 0.12 $\bullet$   & 80.42 $\pm$ 0.24 $\bullet$ 		& 83.26  $\pm$ 0.07 $\bullet$	& \textbf{88.21 $\pm$ 0.04} \\
CIR500G      & 84.77 $\pm$ 0.56 $\circ$   & 77.09 $\pm$ 0.71 $\circ$   & 76.48 $\pm$ 0.81 $\circ$     & 79.03 $\pm$ 0.34 $\circ$   		& 66.38  $\pm$ 0.85 $\bullet$	& 68.41 $\pm$ 0.87 \\
SIN500G      & 79.41 $\pm$ 0.07 $\circ$   & 66.99 $\pm$ 0.10 $\circ$   & 66.81 $\pm$ 0.12 $\circ$     & 74.93 $\pm$ 0.34 $\circ$   		& 62.73  $\pm$ 0.14 $\bullet$	& 65.68 $\pm$ 0.12 \\
STA500G      & 83.97 $\pm$ 0.13 $\bullet$ & 87.43 $\pm$ 0.18 $\bullet$ & 86.89 $\pm$ 0.27 $\bullet$   & 88.26 $\pm$ 0.18 $\bullet$ 		& 85.95  $\pm$ 0.07 $\bullet$	& \textbf{88.60 $\pm$ 0.07} \\  \midrule
1CDT         & 99.77 $\pm$ 0.14 $\bullet$ & 99.92 $\pm$ 0.10 $\ $      & 99.92 $\pm$ 0.10 $\ $        & 99.69 $\pm$ 0.11 $\bullet$  	& 99.56  $\pm$ 0.08 $\bullet$	& \textbf{99.95 $\pm$ 0.04} \\ 
1CHT         & 99.69 $\pm$ 0.20 $\ $      & 99.71 $\pm$ 0.28 $\ $      & 99.56 $\pm$ 0.46 $\ $        & 92.08 $\pm$ 0.22 $\bullet$  	& 99.41  $\pm$ 0.22 $\bullet$	& \textbf{99.86 $\pm$ 0.13} \\ 
UG-2C-2D     & 94.42 $\pm$ 0.12 $\bullet$ & 95.60 $\pm$ 0.12 $\circ$   & 94.36 $\pm$ 0.78 $\bullet$   & 93.98 $\pm$ 0.13 $\bullet$  	& 94.69  $\pm$ 0.13 $\bullet$	& 95.27 $\pm$ 0.09 \\ 
UG-2C-3D     & 93.82 $\pm$ 0.60 $\bullet$ & 95.23 $\pm$ 0.59 $\ $      & 94.61 $\pm$ 0.73 $\bullet$   & 92.94 $\pm$ 0.72 $\bullet$  	& 94.31  $\pm$ 0.69 $\bullet$	& 94.84 $\pm$ 0.59 \\ 
UG-2C-5D     & 90.30 $\pm$ 0.30 $\bullet$ & 92.85 $\pm$ 0.23 $\circ$   & 92.20 $\pm$ 0.23 $\circ$     & 88.21 $\pm$ 0.35 $\bullet$  	& 89.84  $\pm$ 0.38 $\bullet$	& \textbf{91.83 $\pm$ 0.24} \\ 
GEARS-2C-2D  & 95.82 $\pm$ 0.02 $\bullet$ & 95.83 $\pm$ 0.02 $\bullet$ & 95.83 $\pm$ 0.02 $\bullet$   & 94.96 $\pm$ 0.03 $\bullet$  	& 95.03  $\pm$ 0.02 $\bullet$	& \textbf{95.91 $\pm$ 0.01} \\ \midrule
Usenet-1     & 63.76 $\pm$ 2.01 $\bullet$ & 67.26 $\pm$ 3.11 $\bullet$ & 62.11 $\pm$ 2.67 $\bullet$   & 68.02 $\pm$ 1.19 $\bullet$  	& 65.03  $\pm$ 1.70 $\bullet$	& \textbf{73.13 $\pm$ 1.12} \\
Usenet-2     & 72.42 $\pm$ 1.14 $\bullet$ & 68.41 $\pm$ 1.17 $\bullet$ & 70.55 $\pm$ 2.41 $\bullet$   & 72.02 $\pm$ 0.87 $\bullet$  	& 70.56  $\pm$ 1.15 $\bullet$	& \textbf{75.13 $\pm$ 1.06} \\
Luxembourg   & 98.64 $\pm$ 0.00 $\bullet$ & 90.42 $\pm$ 0.55 $\bullet$ & 90.77 $\pm$ 0.52 $\bullet$   & 100.0 $\pm$ 0.00 $\ $       	& 90.99  $\pm$ 0.97 $\bullet$	& 99.98 $\pm$ 0.03   \\
Spam         & 90.79 $\pm$ 0.85 $\bullet$ & 92.18 $\pm$ 0.34 $\bullet$ & 91.78 $\pm$ 0.33 $\bullet$   & 85.53 $\pm$ 1.22 $\bullet$  	& 87.10  $\pm$ 1.45 $\bullet$	& \textbf{95.22 $\pm$ 0.48}    \\
Email        & 74.21 $\pm$ 4.61 $\bullet$ & 72.58 $\pm$ 4.10 $\bullet$ & 60.78 $\pm$ 6.12 $\bullet$   & 83.36 $\pm$ 1.87 $\bullet$  	& 79.83  $\pm$ 3.73 $\bullet$	& \textbf{91.60 $\pm$ 1.86}    \\
Weather      & 75.99 $\pm$ 0.36 $\bullet$ & 70.83 $\pm$ 0.49 $\bullet$ & 70.07 $\pm$ 0.34 $\bullet$   & 68.92 $\pm$ 0.27 $\bullet$  	& 70.21  $\pm$ 0.33 $\bullet$	& \textbf{79.37 $\pm$ 0.26}    \\
GasSensor & 42.36 $\pm$ 3.72 $\bullet$ & 76.61 $\pm$ 0.36 $\bullet$ & 76.61 $\pm$ 0.36 $\bullet$   & 63.82 $\pm$ 3.64 $\bullet$    	& 43.40  $\pm$ 2.88 $\bullet$	& \textbf{81.57 $\pm$ 3.77}    \\
Powersupply  & 74.06 $\pm$ 0.28 $\circ$   & 72.09 $\pm$ 0.29 $\bullet$ & 72.13 $\pm$ 0.23 $\bullet$   & 69.90 $\pm$ 0.38 $\bullet$  	& 68.34  $\pm$ 0.16 $\bullet$	& 72.82 $\pm$ 0.29    \\
Electricity  & 78.97 $\pm$ 0.18 $\bullet$ & 78.03 $\pm$ 0.17 $\bullet$ & 75.62 $\pm$ 0.42 $\bullet$   & 81.05 $\pm$ 0.35 $\bullet$    	& 58.44  $\pm$ 0.71 $\bullet$	& \textbf{84.73 $\pm$ 0.33}    \\
Covertype & 79.08 $\pm$ 1.30 $\bullet$ & 74.17 $\pm$ 0.87 $\bullet$ & 73.13 $\pm$ 1.53 $\bullet$   & 69.43 $\pm$ 1.30 $\bullet$    	& 64.60  $\pm$ 0.89 $\bullet$	& \textbf{89.58 $\pm$ 0.14}    \\\midrule
\textsc{Condor}\ \ \texttt{w/t/l} & \multicolumn{1}{c}{18/ 1/ 3}  & \multicolumn{1}{c}{14/ 3/ 4}  & \multicolumn{1}{c}{17/ 2/ 3}  &  \multicolumn{1}{c}{19/ 1/ 2} & \multicolumn{1}{c}{22/ 0/ 0}  & \multicolumn{1}{c}{ rank first 16/ 22 } \\\bottomrule
\end{tabular}}
\end{table}

\textbf{Real-world Datasets.} We adopt 10 real-world datasets: Usenet-1, Usenet-2, Luxembourg, Spam, Email, Weather, GasSensor, Powersupply, Electricity and Covertype. The number of data items varies from 1,500 to 581,012, and the class number varies from 2 to 6. Detailed descriptions are provided in Appendix~\ref{sec:real-world-info}. We conduct all the experiments for 10 trails and report overall mean and standard deviation of predictive accuracy in Table~\ref{table:accuracy-all}, synthetic datasets are also included. As we can see, \textsc{Condor} has a significant advantage over other comparisons. Actually, it achieves the best on 16 over 22 datasets in total. Besides, it ranks the second on four other datasets. The reason \textsc{Condor} behaves poor on CIR500G and SIN500G is that these two datasets are highly nonlinear (generated by a circle and sine function, respectively.). It is also noteworthy that \textsc{Condor} behaves significant better than other approaches in real-world datasets. These show the superiority of our proposed approach.

\section{Conclusion}
\label{sec:conclusion}
In this paper, a novel and effective approach \textsc{Condor} is proposed for handling concept drift via model reuse, which consists of two key components, $\mathtt{ModelUpdate}$ and $\mathtt{WeightUpdate}$. Our approach is built on a drift detector, when a drift is detected or a maximum accumulation number is achieved, $\mathtt{ModelUpdate}$ leverages and reuses previous models in a weighted manner. Meanwhile, $\mathtt{WeightUpdate}$ adaptively adjusts weights of previous models according to their performance. By the generalization analysis, we prove that the model reuse strategy helps if we properly reuse previous models. Through regret analysis, we show that the weight concentrate on those better-fit models, and the approach achieves a fair dynamic cumulative regret. Empirical results show the superiority of our approach to other comparisons, on both synthetic and real-world datasets.

In the future, it would be interesting to incorporate more techniques from model reuse learning into handling concept drift problems.


\bibliography{Condor-arXiv}

\begin{thebibliography}{}

\bibitem[\protect\citeauthoryear{Bartlett and
  Mendelson}{2002}]{journals/jmlr/BartlettM02}
Peter~L. Bartlett and Shahar Mendelson.
\newblock Rademacher and gaussian complexities: Risk bounds and structural
  results.
\newblock {\em Journal of Machine Learning Research}, 3:463--482, 2002.

\bibitem[\protect\citeauthoryear{Beygelzimer \bgroup \em et al.\egroup
  }{2015}]{conf/icml/BeygelzimerKL15}
Alina Beygelzimer, Satyen Kale, and Haipeng Luo.
\newblock Optimal and adaptive algorithms for online boosting.
\newblock In {\em International Conference on Machine Learning, ICML}, pages
  2323--2331, 2015.

\bibitem[\protect\citeauthoryear{Bifet and
  Gavald{\`{a}}}{2007}]{conf/sdm/BifetG07}
Albert Bifet and Ricard Gavald{\`{a}}.
\newblock Learning from time-changing data with adaptive windowing.
\newblock In {\em Proceedings of the Seventh {SIAM} International Conference on
  Data Mining}, pages 443--448, 2007.

\bibitem[\protect\citeauthoryear{Bousquet \bgroup \em et al.\egroup
  }{2003}]{conf/ac/BousquetBL03}
Olivier Bousquet, St{\'{e}}phane Boucheron, and G{\'{a}}bor Lugosi.
\newblock Introduction to statistical learning theory.
\newblock In {\em Advanced Lectures on Machine Learning, {ML} Summer Schools
  2003}, pages 169--207, 2003.

\bibitem[\protect\citeauthoryear{Bousquet}{2002}]{book/bousquet2002concentration}
Olivier Bousquet.
\newblock {\em Concentration inequalities and empirical processes theory
  applied to the analysis of learning algorithms}.
\newblock PhD thesis, Ecole Polytechnique, 2002.

\bibitem[\protect\citeauthoryear{Cattral \bgroup \em et al.\egroup
  }{2002}]{book/cattral2002evolutionary}
Robert Cattral, Franz Oppacher, and Dwight Deugo.
\newblock Evolutionary data mining with automatic rule generalization.
\newblock 2002.

\bibitem[\protect\citeauthoryear{Cesa-Bianchi and
  Lugosi}{2006}]{book/Cambridge/cesa2006prediction}
Nicolo Cesa-Bianchi and G{\'a}bor Lugosi.
\newblock {\em Prediction, learning, and games}.
\newblock Cambridge university press, 2006.

\bibitem[\protect\citeauthoryear{Cesa{-}Bianchi \bgroup \em et al.\egroup
  }{1997}]{journals/jacm/Cesa-BianchiFHHSW97}
Nicol{\`{o}} Cesa{-}Bianchi, Yoav Freund, David Haussler, David~P. Helmbold,
  Robert~E. Schapire, and Manfred~K. Warmuth.
\newblock How to use expert advice.
\newblock {\em Journal of the ACM}, 44(3):427--485, 1997.

\bibitem[\protect\citeauthoryear{Chen \bgroup \em et al.\egroup
  }{2015}]{journals/archive/chen2015ucr}
Yanping Chen, Eamonn Keogh, Bing Hu, Nurjahan Begum, Anthony Bagnall, Abdullah
  Mueen, and Gustavo Batista.
\newblock The ucr time series classification archive.
\newblock 2015.

\bibitem[\protect\citeauthoryear{Crammer \bgroup \em et al.\egroup
  }{2010}]{conf/colt/CrammerMEV10}
Koby Crammer, Yishay Mansour, Eyal Even{-}Dar, and Jennifer~Wortman Vaughan.
\newblock Regret minimization with concept drift.
\newblock In {\em Annual Conference Computational Learning Theory, COLT}, pages
  168--180, 2010.

\bibitem[\protect\citeauthoryear{de Souza \bgroup \em et al.\egroup
  }{2015}]{conf/sdm/SouzaSGB15}
Vin{\'{\i}}cius M.~A. de~Souza, Diego~Furtado Silva, Jo{\~{a}}o Gama, and
  Gustavo E. A. P.~A. Batista.
\newblock Data stream classification guided by clustering on nonstationary
  environments and extreme verification latency.
\newblock In {\em SIAM International Conference on Data Mining, SDM}, pages
  873--881, 2015.

\bibitem[\protect\citeauthoryear{Du \bgroup \em et al.\egroup
  }{2017}]{conf/nips/DuKSP17}
Simon~S. Du, Jayanth Koushik, Aarti Singh, and Barnab{\'{a}}s P{\'{o}}czos.
\newblock Hypothesis transfer learning via transformation functions.
\newblock In {\em Advances in Neural Information Processing Systems, NIPS},
  pages 574--584, 2017.

\bibitem[\protect\citeauthoryear{Duan \bgroup \em et al.\egroup
  }{2009}]{conf/icml/DuanTXC09}
Lixin Duan, Ivor~W. Tsang, Dong Xu, and Tat{-}Seng Chua.
\newblock Domain adaptation from multiple sources via auxiliary classifiers.
\newblock In {\em International Conference on Machine Learning, ICML}, pages
  289--296, 2009.

\bibitem[\protect\citeauthoryear{Elwell and
  Polikar}{2011}]{journals/tnn/ElwellP11}
Ryan Elwell and Robi Polikar.
\newblock Incremental learning of concept drift in nonstationary environments.
\newblock {\em IEEE Transactions on Neural Networks}, 22(10):1517--1531, 2011.

\bibitem[\protect\citeauthoryear{Forman}{2006}]{conf/sigir/Forman06}
George Forman.
\newblock Tackling concept drift by temporal inductive transfer.
\newblock In {\em International ACM SIGIR Conference on Research \& Development
  in Information Retrieval, SIGIR}, pages 252--259, 2006.

\bibitem[\protect\citeauthoryear{Gama and Kosina}{2014}]{journals/kais/GamaK14}
Jo{\~{a}}o Gama and Petr Kosina.
\newblock Recurrent concepts in data streams classification.
\newblock {\em Knowledge and Information Systems}, 40(3):489--507, 2014.

\bibitem[\protect\citeauthoryear{Gama \bgroup \em et al.\egroup
  }{2003}]{conf/kdd/GamaRM03}
Jo{\~{a}}o Gama, Ricardo Rocha, and Pedro Medas.
\newblock Accurate decision trees for mining high-speed data streams.
\newblock In {\em ACM SIGKDD International Conference on Knowledge Discovery \&
  Data Mining, KDD}, pages 523--528, 2003.

\bibitem[\protect\citeauthoryear{Gama \bgroup \em et al.\egroup
  }{2014}]{journals/csur/GamaZBPB14}
Jo{\~{a}}o Gama, Indre Zliobaite, Albert Bifet, Mykola Pechenizkiy, and
  Abdelhamid Bouchachia.
\newblock A survey on concept drift adaptation.
\newblock {\em ACM Computing Surveys}, 46(4):44:1--44:37, 2014.

\bibitem[\protect\citeauthoryear{Gomes \bgroup \em et al.\egroup
  }{2017}]{journals/csur/GomesBEB17}
Heitor~Murilo Gomes, Jean~Paul Barddal, Fabr{\'{\i}}cio Enembreck, and Albert
  Bifet.
\newblock A survey on ensemble learning for data stream classification.
\newblock {\em ACM Computing Surveys}, 50(2):23:1--23:36, 2017.

\bibitem[\protect\citeauthoryear{Harel \bgroup \em et al.\egroup
  }{2014}]{conf/icml/HarelMEC14}
Maayan Harel, Shie Mannor, Ran El{-}Yaniv, and Koby Crammer.
\newblock Concept drift detection through resampling.
\newblock In {\em International Conference on Machine Learning, ICML}, pages
  1009--1017, 2014.

\bibitem[\protect\citeauthoryear{Harries and
  Wales}{1999}]{journal/harries1999splice}
Michael Harries and New~South Wales.
\newblock Splice-2 comparative evaluation: Electricity pricing.
\newblock {\em Technical Report of South Wales University}, 1999.

\bibitem[\protect\citeauthoryear{Helmbold and
  Long}{1994}]{journals/ml/HelmboldL94}
David~P. Helmbold and Philip~M. Long.
\newblock Tracking drifting concepts by minimizing disagreements.
\newblock {\em Machine Learning}, 14(1):27--45, 1994.

\bibitem[\protect\citeauthoryear{Jaber \bgroup \em et al.\egroup
  }{2013}]{conf/iconip/JaberCT13a}
Ghazal Jaber, Antoine Cornu{\'{e}}jols, and Philippe Tarroux.
\newblock A new on-line learning method for coping with recurring concepts: The
  {ADACC} system.
\newblock In {\em International Conference on Neural Information Processing,
  ICONIP}, pages 595--604, 2013.

\bibitem[\protect\citeauthoryear{Kakade \bgroup \em et al.\egroup
  }{2012}]{journals/jmlr/KakadeST12}
Sham~M. Kakade, Shai Shalev{-}Shwartz, and Ambuj Tewari.
\newblock Regularization techniques for learning with matrices.
\newblock {\em Journal of Machine Learning Research}, 13:1865--1890, 2012.

\bibitem[\protect\citeauthoryear{Katakis \bgroup \em et al.\egroup
  }{2008}]{conf/ecai/KatakisTV08}
Ioannis Katakis, Grigorios Tsoumakas, and Ioannis~P. Vlahavas.
\newblock An ensemble of classifiers for coping with recurring contexts in data
  streams.
\newblock In {\em European Conference on Artificial Intelligence, ECAI}, pages
  763--764, 2008.

\bibitem[\protect\citeauthoryear{Katakis \bgroup \em et al.\egroup
  }{2009}]{journals/jiis/KatakisTBBV09}
Ioannis Katakis, Grigorios Tsoumakas, Evangelos Banos, Nick Bassiliades, and
  Ioannis~P. Vlahavas.
\newblock An adaptive personalized news dissemination system.
\newblock {\em Journal of Intelligent Information Systems}, 32(2):191--212,
  2009.

\bibitem[\protect\citeauthoryear{Katakis \bgroup \em et al.\egroup
  }{2010}]{journals/kais/KatakisTV10}
Ioannis Katakis, Grigorios Tsoumakas, and Ioannis~P. Vlahavas.
\newblock Tracking recurring contexts using ensemble classifiers: an
  application to email filtering.
\newblock {\em Knowledge and Information Systems}, 22(3):371--391, 2010.

\bibitem[\protect\citeauthoryear{Klinkenberg and
  Joachims}{2000}]{conf/icml/KlinkenbergJ00}
Ralf Klinkenberg and Thorsten Joachims.
\newblock Detecting concept drift with support vector machines.
\newblock In {\em International Conference on Machine Learning, ICML}, pages
  487--494, 2000.

\bibitem[\protect\citeauthoryear{Klinkenberg}{2004}]{journals/ida/Klinkenberg04}
Ralf Klinkenberg.
\newblock Learning drifting concepts: Example selection vs. example weighting.
\newblock {\em Intelligent Data Analysis}, 8(3):281--300, 2004.

\bibitem[\protect\citeauthoryear{Kolter and Maloof}{2003}]{conf/icdm/KolterM03}
Jeremy~Z. Kolter and Marcus~A. Maloof.
\newblock Dynamic weighted majority: {A} new ensemble method for tracking
  concept drift.
\newblock In {\em IEEE International Conference on Data Mining, ICDM}, pages
  123--130, 2003.

\bibitem[\protect\citeauthoryear{Kolter and Maloof}{2005}]{conf/icml/KolterM05}
Jeremy~Z. Kolter and Marcus~A. Maloof.
\newblock Using additive expert ensembles to cope with concept drift.
\newblock In {\em International Conference on Machine Learning, ICML}, pages
  449--456, 2005.

\bibitem[\protect\citeauthoryear{Kolter and
  Maloof}{2007}]{journals/jmlr/KolterM07}
J.~Zico Kolter and Marcus~A. Maloof.
\newblock Dynamic weighted majority: An ensemble method for drifting concepts.
\newblock {\em Journal of Machine Learning Research}, 8:2755--2790, 2007.

\bibitem[\protect\citeauthoryear{Koychev}{2000}]{conf/ecai/koychev2000gradual}
Ivan Koychev.
\newblock Gradual forgetting for adaptation to concept drift.
\newblock In {\em In Proceedings of ECAI 2000 Workshop Current Issues in
  Spatio-Temporal Reasoning}, pages 101--106, 2000.

\bibitem[\protect\citeauthoryear{Kuncheva and
  Zliobaite}{2009}]{journals/ida/KunchevaZ09}
Ludmila~I. Kuncheva and Indre Zliobaite.
\newblock On the window size for classification in changing environments.
\newblock {\em Intelligent Data Analysis}, 13(6):861--872, 2009.

\bibitem[\protect\citeauthoryear{Kuzborskij and
  Orabona}{2013}]{conf/icml/KuzborskijO13}
Ilja Kuzborskij and Francesco Orabona.
\newblock Stability and hypothesis transfer learning.
\newblock In {\em International Conference on Machine Learning, ICML}, pages
  942--950, 2013.

\bibitem[\protect\citeauthoryear{Kuzborskij and
  Orabona}{2017}]{journals/mlj/KuzborskijO17}
Ilja Kuzborskij and Francesco Orabona.
\newblock Fast rates by transferring from auxiliary hypotheses.
\newblock {\em Machine Learning}, 106(2):171--195, 2017.

\bibitem[\protect\citeauthoryear{Ledoux and
  Talagrand}{2013}]{book/ledoux2013probability}
Michel Ledoux and Michel Talagrand.
\newblock {\em Probability in Banach Spaces: isoperimetry and processes}.
\newblock Springer Science \& Business Media, 2013.

\bibitem[\protect\citeauthoryear{Lei \bgroup \em et al.\egroup
  }{2015}]{conf/nips/LeiDBK15}
Yunwen Lei, {\"{U}}r{\"{u}}n Dogan, Alexander Binder, and Marius Kloft.
\newblock Multi-class svms: From tighter data-dependent generalization bounds
  to novel algorithms.
\newblock In {\em Advances in Neural Information Processing Systems, NIPS},
  pages 2035--2043, 2015.

\bibitem[\protect\citeauthoryear{Li \bgroup \em et al.\egroup
  }{2013}]{journals/pami/LiTZ13}
Nan Li, Ivor~W. Tsang, and Zhi{-}Hua Zhou.
\newblock Efficient optimization of performance measures by classifier
  adaptation.
\newblock {\em IEEE Transactions on Pattern Analysis and Machine Intelligence},
  35(6):1370--1382, 2013.

\bibitem[\protect\citeauthoryear{Maurer}{2016}]{conf/alt/Maurer16}
Andreas Maurer.
\newblock A vector-contraction inequality for rademacher complexities.
\newblock In {\em International Conference on Algorithmic Learning Theory,
  {ALT}}, pages 3--17, 2016.

\bibitem[\protect\citeauthoryear{Mohri and Medina}{2012}]{conf/alt/MohriM12}
Mehryar Mohri and Andres~Mu{\~{n}}oz Medina.
\newblock New analysis and algorithm for learning with drifting distributions.
\newblock In {\em ALT}, pages 124--138, 2012.

\bibitem[\protect\citeauthoryear{Mohri \bgroup \em et al.\egroup
  }{2012}]{book/MIT/mohri2012foundations}
Mehryar Mohri, Afshin Rostamizadeh, and Ameet Talwalkar.
\newblock {\em Foundations of machine learning}.
\newblock MIT press, 2012.

\bibitem[\protect\citeauthoryear{Schapire and
  Freund}{2012}]{book/MIT/Schapire2012}
Robert~E. Schapire and Yoav Freund.
\newblock {\em Boosting: Foundations and Algorithms}.
\newblock The MIT Press, 2012.

\bibitem[\protect\citeauthoryear{Schapire}{1990}]{journals/ml/Schapire90}
Robert~E. Schapire.
\newblock The strength of weak learnability.
\newblock {\em Machine Learning}, 5:197--227, 1990.

\bibitem[\protect\citeauthoryear{Schlimmer and
  Granger}{1986}]{journals/ml/SchlimmerG86}
Jeffrey~C. Schlimmer and Richard~H. Granger.
\newblock Incremental learning from noisy data.
\newblock {\em Machine Learning}, 1(3):317--354, 1986.

\bibitem[\protect\citeauthoryear{Sch{\"{o}}lkopf \bgroup \em et al.\egroup
  }{2001}]{conf/colt/ScholkopfHS01}
Bernhard Sch{\"{o}}lkopf, Ralf Herbrich, and Alexander~J. Smola.
\newblock A generalized representer theorem.
\newblock In {\em Annual Conference Computational Learning Theory, COLT}, pages
  416--426, 2001.

\bibitem[\protect\citeauthoryear{Segev \bgroup \em et al.\egroup
  }{2017}]{journals/pami/SegevHMCE17}
Noam Segev, Maayan Harel, Shie Mannor, Koby Crammer, and Ran El{-}Yaniv.
\newblock Learn on source, refine on target: {A} model transfer learning
  framework with random forests.
\newblock {\em IEEE Transactions on Pattern Analysis and Machine Intelligence},
  39(9):1811--1824, 2017.

\bibitem[\protect\citeauthoryear{Srebro \bgroup \em et al.\egroup
  }{2010}]{conf/nips/Srebro10}
Nathan Srebro, Karthik Sridharan, and Ambuj Tewari.
\newblock Smoothness, low noise and fast rates.
\newblock In {\em Advances in Neural Information Processing Systems, NIPS},
  pages 2199--2207. 2010.

\bibitem[\protect\citeauthoryear{Street and Kim}{2001}]{conf/kdd/StreetK01}
W.~Nick Street and YongSeog Kim.
\newblock A streaming ensemble algorithm {(SEA)} for large-scale
  classification.
\newblock In {\em ACM SIGKDD International Conference on Knowledge Discovery \&
  Data Mining, KDD}, pages 377--382, 2001.

\bibitem[\protect\citeauthoryear{Sun \bgroup \em et al.\egroup
  }{2018}]{journals/tnnls/suny18}
Y.~Sun, K.~Tang, Z.~Zhu, and X.~Yao.
\newblock Concept drift adaptation by exploiting historical knowledge.
\newblock {\em IEEE Transactions on Neural Networks and Learning Systems}, To
  appear, 2018.

\bibitem[\protect\citeauthoryear{Suykens \bgroup \em et al.\egroup
  }{2002}]{book/2002/suykens2002least}
Johan~AK Suykens, Tony Van~Gestel, and Jos De~Brabanter.
\newblock {\em Least squares support vector machines}.
\newblock World Scientific, 2002.

\bibitem[\protect\citeauthoryear{Tommasi \bgroup \em et al.\egroup
  }{2010}]{conf/cvpr/TommasiOC10}
Tatiana Tommasi, Francesco Orabona, and Barbara Caputo.
\newblock Safety in numbers: Learning categories from few examples with multi
  model knowledge transfer.
\newblock In {\em IEEE Conference on Computer Vision and Pattern Recognition,
  CVPR}, pages 3081--3088, 2010.

\bibitem[\protect\citeauthoryear{Tommasi \bgroup \em et al.\egroup
  }{2014}]{journals/pami/TommasiOC14}
Tatiana Tommasi, Francesco Orabona, and Barbara Caputo.
\newblock Learning categories from few examples with multi model knowledge
  transfer.
\newblock {\em IEEE Transactions on Pattern Analysis and Machine Intelligence},
  36(5):928--941, 2014.

\bibitem[\protect\citeauthoryear{Vergara \bgroup \em et al.\egroup
  }{2012}]{journal/chemistry/vergara2012chemical}
Alexander Vergara, Shankar Vembu, Tuba Ayhan, Margaret~A Ryan, Margie~L Homer,
  and Ram{\'o}n Huerta.
\newblock Chemical gas sensor drift compensation using classifier ensembles.
\newblock {\em Sensors and Actuators B: Chemical}, 166:320--329, 2012.

\bibitem[\protect\citeauthoryear{Vlachos \bgroup \em et al.\egroup
  }{2002}]{conf/kdd/VlachosDGKK02}
Michail Vlachos, Carlotta Domeniconi, Dimitrios Gunopulos, George Kollios, and
  Nick Koudas.
\newblock Non-linear dimensionality reduction techniques for classification and
  visualization.
\newblock In {\em ACM SIGKDD International Conference on Knowledge Discovery \&
  Data Mining, KDD}, pages 645--651, 2002.

\bibitem[\protect\citeauthoryear{Zhou}{2012}]{book/Chapman/zhou2012}
Zhi-Hua Zhou.
\newblock {\em Ensemble Methods: Foundations and Algorithms}.
\newblock Chapman \& Hall/CRC Press, 2012.

\bibitem[\protect\citeauthoryear{Zinkevich}{2003}]{conf/icml/Zinkevich03}
Martin Zinkevich.
\newblock Online convex programming and generalized infinitesimal gradient
  ascent.
\newblock In {\em International Conference on Machine Learning, ICML}, pages
  928--936, 2003.

\bibitem[\protect\citeauthoryear{Zliobaite}{2011}]{journals/ida/Zliobaite11}
Indre Zliobaite.
\newblock Combining similarity in time and space for training set formation
  under concept drift.
\newblock {\em Intelligent Data Analysis}, 15(4):589--611, 2011.

\end{thebibliography}
\bibliographystyle{nips2017}

\appendix
\section{Prerequisite Knowledge and Technical Lemmas}
In this section, we introduce prerequisite knowledge for proving main results technical lemmas. Specifically, we will utilize Rademacher complexity~\citep{journals/jmlr/BartlettM02} in proving generalization bounds. Besides, we will also exploit the function properties when bounding Rademacher complexity and proving the regret bounds. 

\subsection{Rademacher Complexity}
To simplify the presentation, first, we introduce some notations. Let $S$ be a sample of $m$ points drawn i.i.d. according to distribution $\mathcal{D}$, then the risk and empirical risk of hypothesis $h$ are defined by
\begin{equation*}
	R(h) = \mathbb{E}_{(\mathbf{x},y)\sim \mathcal{D}} \left[\ell(h(\mathbf{x}),y)\right], \ \  \hat{R}_S(h) = \frac{1}{m}\sum_{i=1}^m \ell(h(\mathbf{x}_i),y_i).
\end{equation*}

In the following, we will utilize the notion of Rademacher complexity~\citep{journals/jmlr/BartlettM02} to measure the hypothesis complexity and use it to bound the generalization error.
\begin{myDef}(\textit{Rademacher Complexity~\citep{journals/jmlr/BartlettM02}})
	\label{def-rademacher} 
	Let $\mathcal{G}$ be a family of functions and a fixed sample of size $m$ as $S = (\mathbf{z}_1, \cdots, \mathbf{z}_m)$. Then, the \textit{empirical Rademacher complexity} of $\mathcal{G}$ with respect to the sample $S$ is defined as:
	\[
		\hat{\mathfrak{R}}_S(\mathcal{G}) = \mathbb{E}_{\boldsymbol{\sigma}} \left[\sup_{g\in \mathcal{G}} \frac{1}{m} \sum_{i=1}^m \sigma_i g(\mathbf{z}_i)\right].
	\]
	Besides, the \textit{Rademacher complexity} of $\mathcal{G}$ is the expectation of the empirical Rademacher complexity over all samples of size $m$ drawn according to $\mathcal{D}$:
	\begin{equation}
		\label{eq:rademacher} 
		\mathfrak{R}_m(\mathcal{G}) = \mathbb{E}_{S \sim \mathcal{D}^m} [\hat{\mathfrak{R}}_S(\mathcal{G})].
	\end{equation}
\end{myDef}

\subsection{Function Properties}
In this paragraph, we introduce several common and useful functional properties.
\begin{myDef}[Lipschitz Continuity]
A function $f: \mathcal{K} \rightarrow \mathbb{R}$ is $L$-Lipschitz continuous w.r.t. a norm $\lVert \cdot \rVert$ over domain $\mathcal{K}$ if for all $\mathbf{x}, {y} \in \mathcal{K}$, we have
\[ \lvert f({y}) - f(\mathbf{x})\rvert \leq  L\lVert {y}-\mathbf{x}\rVert .\]
\end{myDef}

\begin{myDef}[Strong Convexity]
A function $f: \mathcal{K} \rightarrow \mathbb{R}$ is $\lambda$-strongly convex w.r.t. a norm $\lVert \cdot \rVert$ if for all $\mathbf{x}, {y} \in \mathcal{K}$ and for any $\alpha\in[0,1]$, we have
\begin{equation*}
  \label{eq:sc-1}
  f((1-\alpha)\mathbf{x} + \alpha{y})\leq (1-\alpha)f(\mathbf{x}) + \alpha f({y}) - \frac{\lambda}{2}\alpha(1-\alpha)\lVert \mathbf{x} - {y}\rVert^2.
\end{equation*}
A common and equivalent form for differentiable case is,
\begin{equation}
  \label{eq:sc-2}
  f({y}) \geq f(\mathbf{x}) + \nabla f(\mathbf{x})^\mathrm{T}({y}-\mathbf{x}) + \frac{\lambda}{2}\lVert {y} - \mathbf{x}\rVert^2.
\end{equation}
\end{myDef}

\begin{myDef}[Smoothness]
A function $f: \mathcal{K} \rightarrow \mathbb{R}$ is $\sigma$-smooth w.r.t. a norm $\lVert \cdot \rVert$ if for all $\mathbf{x}, {y} \in \mathcal{K}$, we have
\[ f({y}) \leq f(\mathbf{x}) + \nabla f(\mathbf{x})^\mathrm{T}({y}-\mathbf{x}) + \frac{\sigma}{2}\lVert {y} - \mathbf{x}\rVert^2.\]
If $f$ is differentiable, the above condition is equivalent to a Lipschitz condition over the gradients,
\[ \lVert \nabla f(\mathbf{x}) - \nabla f({y})\rVert \leq \sigma \lVert \mathbf{x} - {y} \rVert. \]
\end{myDef}

\subsection{Technical Lemmas}
To obtain a fast generalization rate, essentially, we need a Bernstein-type concentration inequality. And we adopt the functional generalization of Bennett's inequality due to Bousquet~\citep{book/bousquet2002concentration}, for self-containedness, we state the conclusion in Lemma~\ref{lemma:bousquet-fast-rate} as follow.

\begin{myLemma}[Theorem 2.11 in~\cite{book/bousquet2002concentration}]
\label{lemma:bousquet-fast-rate}
Assume the $X_i$ are identically distributed according to $P$. Let $\mathcal{F}$ be a countable set of functions from $\mathcal{X}$ to $\mathbb{R}$ and assume that all functions $f$ in $\mathcal{F}$ are $P$-measurable, square-integrable and satisfy $\mathbb{E}[f] = 0$. If $\sup_{f\in \mathcal{F}}\esssup f \leq 1$ then we denote 
\[
    Z = \sup_{f\in \mathcal{F}} \sum_{i=1}^n f(X_i),
\]
and if $\sup_{f\in \mathcal{F}} \lVert f \rVert_{\infty} \leq 1$, $Z$ can be defined as above or as
\[
    Z = \sup_{f\in \mathcal{F}} \left\lvert \sum_{i=1}^n f(X_i)\right\rvert.
\]

Let $\sigma$ be a positive real number such that $\sigma^2 \geq \sup_{f\in\mathcal{F}} \Var[f(X_1)]$ almost surely, then for all $x\geq 0$, we have 
\[
    \Pr\left[Z\geq \mathbb{E}[Z] + t\right] \leq \exp\left\{-v g\left(\frac{t}{v}\right)\right\}, 
\]
with $v = n\sigma^2 + 2\mathbb{E}[Z]$ and $g(y) = (1+y)\log(1+y) - y$, also
\[
    \Pr\left[Z\geq \mathbb{E}[Z] + \sqrt{2tv}+\frac{t}{3}\right] \leq e^{-t}.
\]
\end{myLemma}

Besides, for a strongly convex regularizer, we have following property, which will be useful in proving Theorem~\ref{thm:generalization-main-theorem}.
\begin{myLemma}[Corollary 4 in~\cite{journals/jmlr/KakadeST12}]
\label{lemma:kakade-primal-dual}
If $\Omega$ is $\lambda$-strongly convex w.r.t. a norm $\lVert \cdot \rVert$ and $\Omega^\star(\mathbf{0}) = 0$, then, denoting the partial sum $\sum_{j\leq i} \mathbf{v}_j$ by $\mathbf{v}_{1:i}$, we have for any sequence $\mathbf{v}_1,\ldots,\mathbf{v}_n$ and for any $\mathbf{u}$,
\[
    \sum_{i=1}^n \langle \mathbf{v}_i,\mathbf{u}\rangle \leq f^\star(\mathbf{v}_{i:n}) \leq \sum_{i=1}^n \langle\nabla f^\star(\mathbf{v}_{1:t-1}),\mathbf{v}_i\rangle +\frac{1}{2\beta} \sum_{i=1}^n \lVert \mathbf{v}_i\rVert_{\star}^2.
\]
\end{myLemma}

\begin{myLemma}[Lemma 8.1 in~\cite{book/MIT/mohri2012foundations}]
	\label{lemma:rademacher-max}
	Let $\mathcal{F}_1,\ldots,\mathcal{F}_l$ be $l$ hypothesis sets in $\mathbb{R}^\mathcal{X}, l\geq 1$, and let $\mathcal{G} = \{\max\{h_1,\ldots,h_l\}:h_i \in \mathcal{F}_i, i\in [1,l]\}$. Then, for any sample $S$ of size $m$, the empirical Rademacher complexity of $\mathcal{G}$ can be upper bounded as follows:
	\[
		\hat{\mathfrak{R}}_S(\mathcal{G}) \leq \sum_{j=1}^l \hat{\mathfrak{R}}_S(\mathcal{F}_j).
	\]
\end{myLemma}

\begin{proof}
	The proof is based on two observations on $\max$ operator. First one is that $\max\{h_1,\ldots,h_l\} = \max\{h_1,\max\{h_2,\ldots,h_l\}\}$, and thus we can focus on the case when $l=2$. The second observation is that for any $h_1\in \mathcal{F}_1$ and $h_2\in \mathcal{F}_2$, we have
	\[
		\max\{h_1,h_2\} = \frac{1}{2}\left[h_1 + h_2 + \lvert h_1 - h_2\rvert\right].
	\]
	One can refer to page 186 in~\cite{book/MIT/mohri2012foundations} for details. 
\end{proof}

\section{Proof of Theorem~\ref{thm:generalization-main-theorem}}
\label{sec:proof-thm1}
We prove the statement in Theorem~\ref{thm:generalization-main-theorem} in the following two steps,
\begin{enumerate}
    \item[(1)] First, in Lemma~\ref{lemma:generaization-bound-loss}, we establish fast generalization error bound in terms of the Rademacher complexity of loss function family, i.e., $\mathfrak{R}_m(\mathcal{L})$.
    \item[(2)] Next, in Lemma~\ref{lemma:rademacher-of-lipschitz-loss}, we upper bound the Rademacher complexity of Lipschitz loss function family by terms regarding to $R_p$, which is defined as the risk of the combination of previous models $h_p$ on current distribution $\mathcal{D}_k$. 
\end{enumerate} 
\subsection{Fast Rate Generalization Error Bound}

\begin{myLemma}
\label{lemma:generaization-bound-loss}
Assume that the non-negative loss function $\ell:\mathcal{Y}\times \mathcal{Y} \rightarrow \mathbb{R}_+$ is bounded by $M\geq 0$. Let $\mathcal{H}$ denote the hypothesis set and $S$ be a sample of $m$ points drawn i.i.d. according to distribution $\mathcal{D}$. For any constant $r \geq 0$, define the loss function family $\mathcal{L}$ as
\[
    \mathcal{L} = \{ (\mathbf{x},y)\rightarrow \ell(h(\mathbf{x},y):\ h\in \mathcal{H}\ \land\  R(h) \leq r\}.
\]
Then, for any $\delta > 0$, with probability at least $1-\delta$, the following holds for all $h\in \mathcal{H}_r$ with $\mathcal{H}_r = \{h: \ h\in \mathcal{H} \ \land\  R(h) \leq r\}$,
\begin{equation}
    \label{eq:fast-rate}
    R(h) - \hat{R}_S(h) \leq 2 \mathfrak{R}_m(\mathcal{L}) + \frac{3M\log(1/\delta)}{4m} + 3\sqrt{\frac{(8\mathfrak{R}_m(\mathcal{L})+r)M\log(1/\delta)}{4m}}.
\end{equation}
\end{myLemma}

\begin{proof}
The proof is based on the application of functional generalization of Bennett's inequality due to Bousquet~\citep{book/bousquet2002concentration}.

For any sample $S = \{(\mathbf{x}_1,y_1),\ldots,(\mathbf{x}_m,y_m)\}$ and any $h\in \mathcal{H}_r$, we turn to consider the uniform upper bound of $R(h)-\hat{R}_s(h)$. The proof consists of applying functional generalization of Bennett's inequality to function $\Phi$ defined for any sample $S$ by
\[
    \Phi_S = \frac{m}{2M}\sup_{h\in \mathcal{H}_r} \{R(h) - \hat{R}_S(h)\},  
\]
which satisfies the condition in Lemma~\ref{lemma:bousquet-fast-rate}, that is, $\sup_{h\in\mathcal{H}_r} \esssup \frac{1}{2M}\left(\mathbb{E}[\ell(h(\mathbf{x}),y)] - \ell(h(\mathbf{x}_i),y_i)\right) \leq 1$, due to the boundedness of loss function. Now, we can apply Lemma~\ref{lemma:bousquet-fast-rate}, 
\[
	\Pr\left[ \Phi_S \geq \mathbb{E}_S +t\right] \leq \exp\left\{-v g\left(\frac{t}{v}\right)\right\},
\]
where 
\[
	\begin{split}
		v &= m\sigma^2 + 2 \mathbb{E}_S[\Phi_S],\\
		\sigma^2 &= \sup_{h\in \mathcal{H}_r} \Var\left[ \frac{1}{2M} \mathbb{E}_{(\mathbf{x}',y)\sim \mathcal{D}}[\ell(h(\mathbf{x}'),y')]-\ell(h(\mathbf{x}),y)\right].
	\end{split}
\]

And we can reverse the above inequality and obtain that for any $\delta>0$, with probability at least $1-\delta$, the following holds for all $h\in \mathcal{H}_r$,
\begin{equation}
	\label{eq:concentration}
	\Phi_S \leq \mathbb{E}_S[\Phi_S] + \frac{3\log(1/\delta)}{4} + \frac{3}{2}\sqrt{v\log(1/\delta)}.
\end{equation}
Thus, we proceed to bound term $\sigma^2$ and $\mathbb{E}_S[\Phi_S]$. First, we bound the term $\sigma^2$.
\begin{align}\nonumber
		\sigma^2 &= \sup_{h\in \mathcal{H}_r} \Var\left[ \frac{1}{2M} \mathbb{E}_{(\mathbf{x}',y)\sim \mathcal{D}}[\ell(h(\mathbf{x}'),y')]-\ell(h(\mathbf{x}),y)\right]\\\nonumber
		& = \sup_{h\in \mathcal{H}_r} \mathbb{E}_{(\mathbf{x},y)\sim \mathcal{D}}\left[ \frac{1}{4M^2}\left(\ell(h(\mathbf{x}),y)-\mathbb{E}_{(\mathbf{x}',y')\sim \mathcal{D}}[\ell(h(\mathbf{x}'),y')]\right)^2\right]\\\nonumber
		\label{eq:bound-variance-1}
		& \leq \sup_{h\in \mathcal{H}_r} \frac{1}{4M^2} \mathbb{E}_{(\mathbf{x},y)\sim \mathcal{D}} [\ell(h(\mathbf{x}),y)^2]\\
		& \leq \sup_{h\in \mathcal{H}_r} \frac{1}{4M} \mathbb{E}_{(\mathbf{x},y)\sim \mathcal{D}} [\ell(h(\mathbf{x}),y)]\\		
		\label{eq:bound-variance-2}
		& = \sup_{h\in \mathcal{H}_r} \frac{1}{4M} R(h) = \frac{r}{4M}.
\end{align}

Eq.~\eqref{eq:bound-variance-1} holds due to the boundedness of loss function, and Eq.~\eqref{eq:bound-variance-2} holds because of the definition of $H_r$, as we know $R(h)\leq r$ holds for all $h\in \mathcal{H}_r$. 

Next, we turn to prove the bound on $\mathbb{E}_S[\Phi_S]$ by utilizing the standard \emph{symmetrization} technique~\citep{conf/ac/BousquetBL03,book/MIT/mohri2012foundations}.
\begin{align}
		\nonumber
		\mathbb{E}_S[\Phi_S] &= \mathbb{E}_S\left[\frac{m}{2M}\sup_{h\in \mathcal{H}_r} \{R(h) - \hat{R}_S(h)\}\right]\\
		\nonumber
		&= \mathbb{E}_S\left[\frac{m}{2M}\sup_{h\in \mathcal{H}_r} \left\{\mathbb{E}_{S'}[\hat{R}_{S'}(h) - \hat{R}_S(h)]\right\}\right]\\
		\label{eq:symmertrization-1}
		&\leq \frac{m}{2M}\mathbb{E}_{S,S'}\left[\sup_{h\in \mathcal{H}_r} \left\{\hat{R}_{S'}(h) - \hat{R}_S(h)\right\}\right]\\
		\nonumber	
		&= \frac{m}{2M}\mathbb{E}_{S,S'}\left[\sup_{h\in \mathcal{H}_r} \frac{1}{m} \sum_{i=1}^m \big( \ell(h(\mathbf{x}'_i),y'_i) - \ell(h(\mathbf{x}_i),y_i)\big)\right]\\	
		\label{eq:symmertrization-2}
		&= \frac{m}{2M}\mathbb{E}_{\boldsymbol{\sigma},S,S'}\left[\sup_{h\in \mathcal{H}_r} \frac{1}{m} \sum_{i=1}^m \sigma_i\big( \ell(h(\mathbf{x}'_i),y'_i) - \ell(h(\mathbf{x}_i),y_i)\big)\right]\\
		\label{eq:symmertrization-3}	
		&\leq \frac{m}{2M}\mathbb{E}_{\boldsymbol{\sigma},S'}\left[\sup_{h\in \mathcal{H}_r} \frac{1}{m} \sum_{i=1}^m \sigma_i\big( \ell(h(\mathbf{x}'_i),y'_i))\big)\right] + \frac{m}{2M}\mathbb{E}_{\boldsymbol{\sigma},S}\left[\sup_{h\in \mathcal{H}_r} \frac{1}{m} \sum_{i=1}^m -\sigma_i\big( \ell(h(\mathbf{x}_i),y_i))\big)\right]\\	
		\nonumber
		& = \frac{m}{M}\mathbb{E}_{\boldsymbol{\sigma},S}\left[\sup_{h\in \mathcal{H}_r} \frac{1}{m} \sum_{i=1}^m \sigma_i\big( \ell(h(\mathbf{x}_i),y_i))\big)\right] = \frac{m}{M} \mathfrak{R}_m(\mathcal{L}).
\end{align}
\eqref{eq:symmertrization-1} holds due to the convexity of supremum function, and thus we can apply Jensen's inequality. In \eqref{eq:symmertrization-2}, we introduce Rademacher random variables $\sigma_1$s, that are uniformly distributed independent random variables taking values in $\{-1,+1\}$, and thus this does not change the expectation. \eqref{eq:symmertrization-2} holds due to the sub-additivity of supremum function.

Thus, we can obtain the bound on $v$,
\[
	v = m\sigma^2 +2\mathbb{E}_S[\Phi_S] \leq \frac{rm}{4M} + \frac{2m}{M}\mathfrak{R}_m(\mathcal{L}).
\]

Combining the bounds on $v$ and $\mathbb{E}_S[\Phi_S]$, we have

\begin{equation*}
	\begin{split}
	\Phi_S &= \frac{m}{2M}\sup_{h\in \mathcal{H}_r} \{R(h) - \hat{R}_S(h)\}\\
	 & \leq \frac{m}{M} \mathfrak{R}_m(\mathcal{L}) + \frac{3\log(1/\delta)}{4} + \frac{3}{2}\sqrt{\left(\frac{rm}{4M} + \frac{2m}{M}\mathfrak{R}_m(\mathcal{L})\right)\log(1/\delta)}.
	\end{split}
\end{equation*}

Hence, we complete the proof that for any hypothesis $h\in \mathcal{H}_r$, we have 
\[
	\begin{split}
		R(h) - \hat{R}_S(h) &\leq \sup_{h\in \mathcal{H}_r} \{R(h) - \hat{R}_S(h)\} = (2M/m) \Phi_S\\
		& \leq 2 \mathfrak{R}_m(\mathcal{L}) + \frac{3M\log(1/\delta)}{4m} + 3\sqrt{\frac{(r+8\mathfrak{R}_m(\mathcal{L}))M\log(1/\delta)}{4m}}.
	\end{split}
\]
\end{proof} 

\subsection{Rademacher Complexity of Lipschitz Loss Function Family}
\begin{myLemma}
\label{lemma:rademacher-of-lipschitz-loss}
Assume that the non-negative loss function $\ell:\mathcal{Y}\times \mathcal{Y} \rightarrow \mathbb{R}_+$ is bounded by $M\geq 0$, and is $L$-Lipschitz continuous. Given the model pool $\{h_1,h_2,\ldots,h_p\}$ with $\mathbf{h}_p(\mathbf{x}):= [h_1(\mathbf{x}),h_2(\mathbf{x}),\ldots,h_p(\mathbf{x})]^\mathrm{T}$, and $h_p$ is a linear combination of previous models, i.e., $h_p(\mathbf{x}) = \langle \boldsymbol{\beta},\mathbf{f}(\mathbf{x})\rangle$. Define the feasible set $\mathcal{W} = \{\mathbf{w}:\Omega(\mathbf{w}) \leq \alpha \hat{R}_S(h_p)\}$ and $\mathcal{V} = \{\boldsymbol{\beta}:\Omega(\boldsymbol{\beta}) \leq \rho\}$, and let the loss function family $\mathcal{L}$ be as,
\[
	\mathcal{L} = \{ (\mathbf{x},y) \rightarrow \ell(\langle \mathbf{w},\mathbf{x}\rangle + h_p(\mathbf{x}),y): \ \mathbf{w} \in \mathcal{W} \ \land \ \boldsymbol{\beta} \in \mathcal{V}\},
\]
and let $S$ be a sample of $m$ points drawn i.i.d. according to distribution $\mathcal{D}$, then we have
\begin{equation}
	\mathfrak{R}_m(\mathcal{L}) \leq 2L \sqrt{\frac{B^2\alpha R_p+C^2\rho}{\lambda m}}.
\end{equation}
where $R_p = R(h_p)$, $B = \sup_{\mathbf{x}\in \mathcal{X}} \lVert \mathbf{x}\rVert_\star$ and $C = \sup_{\mathbf{x}\in \mathcal{X}} \lVert \mathbf{h}_p(\mathbf{x})\rVert_\star$.
\end{myLemma}

\begin{proof}
Let $\mathcal{L}$ be the loss function family associated with the hypothesis set $\mathcal{H}$. Namely, $\mathcal{L}=\{(\mathbf{x},y)\rightarrow \ell(h(\mathbf{x}),y): h\in \mathcal{H}\}$, where $\mathcal{H} =\{h_{\mathbf{w},p}: \ \mathbf{w} \in \mathcal{W} \ \land \ \boldsymbol{\beta}\in \mathcal{V}\}$. The empirical Rademacher complexity of $\mathcal{L}$ is,
\begin{equation}
	\label{eq:rademacher-expansion}
	\begin{split}
	\hat{\mathfrak{R}}_S(\mathcal{L}) &= \frac{1}{m} \mathbb{E}_{\boldsymbol{\sigma}}\left[\sup_{h\in \mathcal{H}} \sum_{i=1}^m \sigma_i \ell(h(\mathbf{x}_i),y_i)\right]\\
	& = \frac{1}{m} \mathbb{E}_{\boldsymbol{\sigma}}\left[\sup_{\mathbf{w}\in \mathcal{W}} \sum_{i=1}^m \sigma_i \ell\left(\langle \mathbf{w},\mathbf{x}_i\rangle +h_p(\mathbf{x}_i),y_i\right)\right]\\
	& \leq 	\frac{L}{m} \mathbb{E}_{\boldsymbol{\sigma}}\left[\sup_{\mathbf{w}\in \mathcal{W}} \sum_{i=1}^m \sigma_i \langle \mathbf{w},\mathbf{x}_i\rangle \right] + \frac{L}{m} \mathbb{E}_{\boldsymbol{\sigma}}\left[\sup_{\boldsymbol{\beta}\in \mathcal{V}} \sum_{i=1}^m \sigma_i \langle \boldsymbol{\beta},\mathbf{h}_p(\mathbf{x}_i)\rangle \right]\\
	\end{split}	
\end{equation}

The last inequality holds due to the fact that loss function $\ell$ is a $L$-Lipschitz continuous, and thus we can apply the \emph{Talagrand's Comparison Inequality}~\citep{book/ledoux2013probability} to relate the Rademacher complexity of loss function family and that of hypothesis set.

We turn to bound the two empirical Rademacher complexity terms in the r.h.s. of above inequality. Similar to the technique in~\cite{journals/jmlr/KakadeST12}, we utilize the primal-dual property of strongly-convex regularizer, and introduce the variable $t_1$ and $t_2$ into two Rademacher complexity terms separately, with aim to obtain a tighter bound by tunning the variables.

For the first term in the r.h.s. of~\eqref{eq:rademacher-expansion}, we apply Lemma~\ref{lemma:kakade-primal-dual} with $\mathbf{u} = \mathbf{w}$ and $\mathbf{v}_i = t_1 \mathbf{x}_i$,
\[
	\begin{split}
		& \mathbb{E}_{\boldsymbol{\sigma}}\left[\sup_{\mathbf{w}\in \mathcal{W}} \frac{1}{m} \sum_{i=1}^m  \langle \mathbf{w},t_1\sigma_i\mathbf{x}_i\rangle \right]\\
		\leq & \frac{1}{m}\mathbb{E}_{\boldsymbol{\sigma}}\left[ \frac{t_1^2}{2\lambda} \sum_{i=1}^m \lVert \sigma_i \mathbf{x}_i\rVert_\star^2 + \sup_{\mathbf{w} \in \mathcal{W}} \Omega(\mathbf{w}) + \sum_{i=1}^m \langle \nabla\Omega^\star(\mathbf{v}_{1:i-1}),\sigma_i\mathbf{x}_i \rangle\right]\\
		\leq & \frac{B^2 t_1^2}{2\lambda} +\frac{\alpha \hat{R}_S(h_p)}{m},
	\end{split}
\]
where the last inequality holds due to the fact that $\sup_{\mathbf{w} \in \mathcal{W}}\Omega(\mathbf{w}) \leq \alpha \hat{R}_S(h_p)$ and $\sup_{\mathbf{x}\in \mathcal{X}} \lVert \mathbf{x}\rVert_\star \leq B$. By dividing $t_1$ on both sides, and notice that the above upper bound holds for any $t_1>0$, we choose a particular $t_1$ making the upper bound tight,
\begin{equation}
	\label{eq:rademacher-term1}
	\frac{L}{m} \mathbb{E}_{\boldsymbol{\sigma}}\left[\sup_{\mathbf{w}\in \mathcal{W}} \sum_{i=1}^m \sigma_i \langle \mathbf{w},\mathbf{x}_i\rangle \right] \leq \inf_{t_1 > 0} L\left(\frac{B^2t_1}{2\lambda} +\frac{\alpha \hat{R}_S(h_p)}{mt_1}\right) = L\sqrt{\frac{2B^2\alpha \hat{R}_S(h_p)}{\lambda m}}.
\end{equation}

Similarly, we introduce variable $t_2$ as tunning parameter for the second term in the r.h.s. of~\eqref{eq:rademacher-expansion} with $\mathbf{u} = \mathbf{w}$ and $\mathbf{v}_i = t_2 \mathbf{h}_p(\mathbf{x}_i)$,
\[
	\begin{split}
	& \mathbb{E}_{\boldsymbol{\sigma}}\left[\sup_{\boldsymbol{\beta}\in \mathcal{V}} \frac{1}{m}\sum_{i=1}^m \langle \boldsymbol{\beta},t_2\sigma_i\mathbf{h}_p(\mathbf{x}_i) \rangle\right]\\
	\leq & \frac{1}{m} \mathbb{E}_{\boldsymbol{\sigma}}\left[\frac{t_2^2}{2\lambda} \sum_{i=1}^m \lVert \sigma_i \mathbf{h}_p(\mathbf{x}_i)\rVert^2_\star + \sup_{\boldsymbol{\beta} \in \mathcal{V}}\Omega(\boldsymbol{\beta}) + \sum_{i=1}^m \langle \nabla\Omega^\star(\mathbf{v}_{1:i-1}),\sigma_i\mathbf{h}_p(\mathbf{x}_i) \rangle \right]\\
	\leq & \frac{t_2^2C^2}{2\lambda} +\frac{\rho}{m},
	\end{split}
\]
where the last inequality holds due to the fact that $\sup_{\boldsymbol{\beta}\in \mathcal{V}}\Omega(\boldsymbol{\beta}) \leq \rho$ and $\sup_{\mathbf{x}\in \mathcal{X}} \lVert \mathbf{h}_p(\mathbf{x})\rVert_\star \leq C$. By dividing $t_2$ on both sides, and notice that the above upper bound holds for any $t_2>0$, we choose a particular $t_2$ making the upper bound tight,
\begin{equation}
	\label{eq:rademacher-term2}
	\frac{L}{m} \mathbb{E}_{\boldsymbol{\sigma}}\left[\sup_{\boldsymbol{\beta}\in \mathcal{V}} \sum_{i=1}^m \sigma_i \langle \boldsymbol{\beta},\mathbf{h}_p(\mathbf{x}_i)\rangle \right] \leq \inf_{t_2 > 0} L\left(\frac{t_2C^2}{2\lambda} +\frac{\rho}{mt_2}\right) = L\sqrt{\frac{2C^2\rho}{\lambda m}}.	
\end{equation}

Combing~\eqref{eq:rademacher-term1} and~\eqref{eq:rademacher-term2}, we obtain the bound for $\hat{\mathfrak{R}}_S(\mathcal{L})$. Notice that the square-root function is concave, by applying the Jensen's inequality w.r.t. the both side, we have
\[
	\begin{split}
	\mathfrak{R}_m(\mathcal{L}) &= \mathbb{E}_S[\hat{\mathfrak{R}}_S(\mathcal{L})] \leq L\mathbb{E}_S\left[\sqrt{\frac{2B^2\alpha \hat{R}_S(h_p)}{\lambda m}} + \sqrt{\frac{2C^2\rho}{\lambda m}}\right] \\
	&\leq L\left(\sqrt{\frac{2B^2\alpha \mathbb{E}_S\left[\hat{R}_S(h_p)\right]}{\lambda m}} + \sqrt{\frac{2C^2\rho}{\lambda m}}\right) \\
	& = L \left( \sqrt{\frac{2B^2\alpha R(h_p)}{\lambda m}} + \sqrt{\frac{2C^2\rho}{\lambda m}}\right)\\
	& \leq 2L \sqrt{\frac{B^2\alpha R_p+C^2\rho}{\lambda m}}
	\end{split}
\]
The last inequality holds due to the fact that $\sqrt{a} + \sqrt{b} \leq \sqrt{2(a+b)}$, for any $a,b \geq 0$. Besides, $R_p$ is short for $R(h_p)$. Hence, we complete the proof of the statement.
\end{proof}

\subsection{Proof of Theorem~\ref{thm:generalization-main-theorem}}

\begin{proof}
To prove the generalization bound stated in Theorem~\ref{thm:generalization-main-theorem}, we turn to prove the uniform bound over the following hypothesis set 
\[\mathcal{H} = \{\mathbf{x} \rightarrow \langle \mathbf{w},\mathbf{x}\rangle + h_p(\mathbf{x}):\ \Omega(\mathbf{w}) \leq \frac{1}{\lambda}\hat{R}(h_p)\  \land\  \hat{R}(h_{\mathbf{w},p}) \leq \hat{R}(h_p) \},\]
where $h_{\mathbf{w},p}(\mathbf{x}) = \langle \mathbf{w},\mathbf{x}\rangle + h_p(\mathbf{x})$.

First, we verify that the model $\hat{h}_S$ returned by $\mathtt{ModelUpdate}$ procedure belongs to the set $\mathcal{H}$. Since $\hat{\mathbf{w}}$ is optimal solution, it is apparently better than the choice $\mathbf{0}$,
\[
    \hat{R}_S(\hat{h}_S) + \lambda \Omega(\hat{\mathbf{w}}) \leq \hat{R}_S(h_p) + \lambda \Omega(\mathbf{0}) = \hat{R}_S(h_p),
\]
combining the non-negative property of loss and regularizer, we know that $\hat{h}_S$ belongs to the $\mathcal{H}$.

Besides, we can upper bound the Rademacher complexity of any hypothesis in the hypothesis set $\mathcal{H}$ as,
\begin{equation}
	\label{eq:upper-bound-r}
    r = \sup_{h\in \mathcal{H}} R(h) = \sup_{h\in \mathcal{H}} \mathbb{E}_S [\hat{R}_S(h)] \leq \mathbb{E}_S[\sup_{h\in \mathcal{H}}\hat{R}_S(h)] \leq \mathbb{E}_S[\hat{R}_S(h_p)] = R_p.
\end{equation}

Thus, we can apply Lemma~\ref{lemma:generaization-bound-loss} by setting $r = R_p$, obtaining that for any $\delta > 0$, with probability at least $1-\delta$, the following holds for all $h\in \mathcal{H}$,

\begin{equation}
    \label{eq:rademacher-bound}
    R(h) - \hat{R}_S(h) \leq 2 \mathfrak{R}_m(\mathcal{L}) + \frac{3M\log(1/\delta)}{4m} + 3\sqrt{\frac{(8\mathfrak{R}_m(\mathcal{L})+R_p)M\log(1/\delta)}{4m}}.
\end{equation}

To bound the Rademacher complexity term $\mathfrak{R}_m(\mathcal{L})$, we can apply Lemma~\ref{lemma:rademacher-of-lipschitz-loss} by setting $\alpha = 1/\lambda$, 
\[
    \mathfrak{R}_m(\mathcal{L}) \leq 2L\sqrt{\frac{B^2R_p + C^2\lambda \rho}{\lambda^2 m}}.  
\]

This in conjunction with~\eqref{eq:rademacher-bound} yields~\eqref{eq:main-results}.
\end{proof}

\begin{myCor}
    \label{corollary:smooth-loss}
    If further suppose an $H$-smooth condition on loss function, and under the same conditions (except $L$-Lipschitz continuity) stated in Theorem~\ref{thm:generalization-main-theorem}. Then, for any $\delta > 0$, with probability at least $1-\delta$, the following holds,
    \begin{equation*}
    \begin{split}
		R(h_k) - \hat{R}(h_k) = O\left( \frac{1}{\sqrt{m}}\left( \sqrt{R_p} + \sqrt{\frac{\sqrt{H} R_p}{\lambda^2}} +\sqrt[4]{\frac{\sqrt{H} R_p}{\lambda^2 m}}\right) + \frac{1}{m}\left(\sqrt{\frac{1}{\lambda}} + \sqrt[4]{\frac{1}{\lambda}}\right)\right).	
    \end{split}
\end{equation*}

\end{myCor}
\begin{proof}
From Lemma B.1 in~\cite{conf/nips/Srebro10}, we know that for any $H$-smooth non-negative function $f: \mathbb{R} \rightarrow \mathbb{R}_+$ and any $x,y\in \mathbb{R}$,
\[
    \lvert f(x) - f(y)\rvert \leq \sqrt{6H(f(x) + f(y))}\lvert x - y\rvert.
\]

Suppose $f$ is bounded by $M>0$, we can easily conclude that $f$ is also $L$-Lipschitz continuous with $L=\sqrt{12HM}$. Thus, we can apply Theorem~\ref{thm:generalization-main-theorem} to obtain a generalization bound. It turns out above bound can be sometimes tighter than that obtained by directly analyzing the smoothness condition in Theorem 1~\citep{journals/mlj/KuzborskijO17}.  
\end{proof}
\begin{myRemark}
It is noteworthy to mention that we deal with the \emph{non-convex} formulation. Essentially, the loss is not necessarily convex in our analysis. We only assume a bounded and Lipschitz condition (in Theorem~\ref{thm:generalization-main-theorem}) or smooth condition (in Corollary~\ref{corollary:smooth-loss}) for the loss function, along with strongly convexity condition for regularizer. These two conditions can be easily satisfied by common models. For example, in SVMs we use $\ell_2$ regularization, which is $2$-strongly convex, and the hinge loss $\ell(z,y)=[1-yz]_+$, which is $1$-Lipschitz. 
\end{myRemark}

\begin{myRemark}
The main techniques in the proof are inspired by~\cite{journals/mlj/KuzborskijO17}, but we differ in three aspects. First, we only assume the Lipschitz condition for loss function, and do not assume and utilize the smoothness as~\cite{journals/mlj/KuzborskijO17}. Second, our analysis is carried out more scrutinizingly but with simpler proof, and can be easily extended to the smooth loss function case (Corollary~\ref{corollary:smooth-loss}). Our analysis obtain a slightly better bound by a constant factor, though in the same order with theirs. Moreover, our analysis is extended into multi-class scenarios in Section~\ref{sec:multi-class-model-reuse}.
\end{myRemark}

\section{Multi-Class Model Reuse Learning}
\label{sec:multi-class-model-reuse}
In this section, we extend the generalization analysis from binary model reuse learning to multi-class case, and provide theoretical analysis and corresponding proofs.

\subsection{Rademacher Complexity of Lipschitz Loss Function Family in Multi-Class Case}
\begin{myLemma}
\label{lemma:rademacher-of-lipschitz-loss-mc}
Let $H\subseteq \mathbb{R}^{\mathcal{X}\times \mathcal{Y}}$ be a hypothesis set with $\mathcal{Y} = \{1,2,\ldots,c\}$. Assume that the non-negative loss function $\ell:\mathbb{R} \rightarrow \mathbb{R}_+$ is $L$-regular loss function. Given a set of models $\{h_1,h_2,\ldots,h_p\}$ with $\mathbf{h}_p(\mathbf{x}):= [h_1(\mathbf{x}),h_2(\mathbf{x}),\ldots,h_p(\mathbf{x})]^\mathrm{T}$, and $h_p$ is a linear combination of previous models, i.e., $h_p(\mathbf{x}) = \langle \boldsymbol{\beta},\mathbf{f}(\mathbf{x})\rangle$. Define the feasible set $\mathcal{W} = \{\mathbf{w}:\Omega(\mathbf{w}) \leq \alpha \hat{R}_S(h_p)\}$ and $\mathcal{V} = \{\boldsymbol{\beta}:\Omega(\boldsymbol{\beta}) \leq \rho\}$, let the loss function family associated with hypothesis set $\mathcal{H}$ is defined as 
\[
	\mathcal{L} = \big\{(\mathbf{x},y) \rightarrow \ell\left(\rho_{h_{\mathbf{w},p}}(\mathbf{x},y)\right): \ \mathbf{w} \in \mathcal{W} \ \land \ \boldsymbol{\beta} \in \mathcal{V}\big\},
\]
and let $S$ be a sample of $m$ points drawn i.i.d. according to distribution $\mathcal{D}$, then we have
\begin{equation}
	\mathfrak{R}_m(\mathcal{L}) \leq 2L c^2 \sqrt{\frac{B^2\alpha R_p + C^2\lambda\rho}{\lambda m}}.
\end{equation}
where $R_p = R(h_p)$, $B = \sup_{\mathbf{x}\in \mathcal{X}} \lVert \mathbf{x}\rVert_\star$ and $C = \sup_{\mathbf{x}\in \mathcal{X}} \lVert \mathbf{h}_p(\mathbf{x})\rVert_\star$.
\end{myLemma}

\begin{proof}
The loss function family associated with hypothesis set $\mathcal{H}$ is defined as 
\[
	\mathcal{L} = \big\{(\mathbf{x},y) \rightarrow \ell\left(\rho_h(\mathbf{x},y)\right): \ h\in \mathcal{H} \ \land \ R(h)\leq r\big\}.
\]

And the empirical Rademacher complexity of $\mathcal{L}$ can be calculated as,
\begin{align}
	\nonumber
	\hat{\mathfrak{R}}_S(\mathcal{L}) & = \frac{1}{m} \mathbb{E}_{\boldsymbol{\sigma}}\left[ \sup_{h\in \mathcal{H}} \sum_{i=1}^m \sigma_i \ell\big(\rho_h(\mathbf{x}_i,y_i)\big)\right]\\
	\label{eq:rademacher-mc-1}
	& \leq \frac{L}{m} \mathbb{E}_{\boldsymbol{\sigma}}\left[ \sup_{h\in \mathcal{H}} \sum_{i=1}^m \sigma_i \rho_{h_{W,p}}(\mathbf{x}_i,y_i)\right]\\
	\label{eq:rademacher-mc-2}
	& \leq \frac{L}{m} \sum_{y\in \mathcal{Y}}\mathbb{E}_{\boldsymbol{\sigma}}\left[ \sup_{h\in \mathcal{H}} \sum_{i=1}^m \sigma_i \rho_{h_{W,p}}(\mathbf{x}_i,y)\right]\\
	\label{eq:rademacher-mc-3}
	& = \frac{L}{m} \sum_{y\in\mathcal{Y}}\left\{\mathbb{E}_{\boldsymbol{\sigma}}\left[ \sup_{h\in \mathcal{H}} \sum_{i=1}^m \sigma_i \left(h_{W,p}(\mathbf{x}_i,y) - \max_{y'\neq y} h_{W,p}(\mathbf{x}_i,y') \right) \right]\right\}\\
	\label{eq:rademacher-mc-4}	
	& \leq \frac{L}{m} \sum_{y\in \mathcal{Y}}\left\{ \mathbb{E}_{\boldsymbol{\sigma}}\left[ \sup_{h\in \mathcal{H}} \sum_{i=1}^m \sigma_i h_{W,p}(\mathbf{x}_i,y)\right] + \mathbb{E}_{\boldsymbol{\sigma}}\left[ \sup_{h\in \mathcal{H}} \sum_{i=1}^m \sigma_i \max_{y'\in \mathcal{Y}\setminus  y} h_{W,p}(\mathbf{x}_i,y') \right]\right\}\\
	\label{eq:rademacher-mc-5}	
	& \leq \frac{L}{m} \sum_{y\in \mathcal{Y}}\left\{\mathbb{E}_{\boldsymbol{\sigma}}\left[ \sup_{W\in \mathcal{W}} \sum_{i=1}^m \sigma_i h_{W,p}(\mathbf{x}_i,y)\right] + \sum_{y'\in \mathcal{Y}\setminus  y}\mathbb{E}_{\boldsymbol{\sigma}}\left[ \sup_{W\in \mathcal{W}} \sum_{i=1}^m \sigma_i h_{W,p}(\mathbf{x}_i,y')\right]\right\}\\
	\label{eq:rademacher-mc-6}	
	& = \frac{Lc^2}{m} \mathbb{E}_{\boldsymbol{\sigma}}\left[ \sup_{W\in \mathcal{W}} \sum_{i=1}^m \sigma_i \Big(\langle \mathbf{w}_y,\mathbf{x}_i\rangle + \langle \boldsymbol{\beta},\mathbf{h}_p(\mathbf{x}_i)\rangle\Big)\right]\\
	\label{eq:rademacher-mc-7}	
	& \leq \underbrace{\frac{Lc^2}{m} \mathbb{E}_{\boldsymbol{\sigma}}\left[ \sup_{W\in \mathcal{W}} \sum_{i=1}^m \langle \mathbf{w}_y,\sigma_i \mathbf{x}_i\rangle\right]}_{\texttt{term (a)}} + \underbrace{\frac{Lc^2}{m} \mathbb{E}_{\boldsymbol{\sigma}}\left[ \sup_{\boldsymbol{\beta}\in \mathcal{V}}\langle \boldsymbol{\beta},\sigma_i \mathbf{h}_p(\mathbf{x}_i)\rangle\right]}{\texttt{term (b)}}
\end{align}

Here, \eqref{eq:rademacher-mc-1} holds due to the fact that loss function $\ell$ is $L$-Lipschitz, and thus we can apply Talagrand's Comparison Inequality~\citep{book/ledoux2013probability,book/MIT/mohri2012foundations} to relate the Rademacher complexity of loss family to margin function family. \eqref{eq:rademacher-mc-2}, \eqref{eq:rademacher-mc-4} and \eqref{eq:rademacher-mc-7} hold due to the sub-additivity of supremum function. \eqref{eq:rademacher-mc-3} is an application of margin definition. \eqref{eq:rademacher-mc-5} holds due to Lemma~\ref{lemma:rademacher-max}.

In the following, we proceed to bound two terms in \eqref{eq:rademacher-mc-7}, i.e., \texttt{term (a)} and \texttt{term (b)}. The basic idea is essentially the same as the technique in \eqref{eq:rademacher-term1} and \eqref{eq:rademacher-term2}, a slight trick here is that we need to introduce an auxiliary vector regularizer, the Euclidean norm $\Omega_{\text{aux}}(\mathbf{w}) = \frac{1}{2}\lVert \mathbf{w} \rVert^2$, to establish the connection between Frobenius norm of matrix and Lemma~\ref{lemma:kakade-primal-dual}, since Lemma~\ref{lemma:kakade-primal-dual} is designed for vector norm.

Similarly, we introduce the variable $t_1$ and $t_2$ into two Rademacher complexity term, i.e., \texttt{term (a)} and \texttt{term (b)}, separately. Note that the Euclidean norm $\Omega_{\text{aux}}(\mathbf{w}) = \lVert \mathbf{w} \rVert^2$ is 1-strongly convex function and its dual norm is itself.

\begin{align}
	\nonumber
	& \frac{Lc^2}{m} \mathbb{E}_{\boldsymbol{\sigma}}\left[ \sup_{W\in \mathcal{W}} \sum_{i=1}^m \langle \mathbf{w}_y,t_1\sigma_i \mathbf{x}_i\rangle\right]\\
	\label{eq:bound-term-(a)-mc-1}
	\leq & \frac{Lc^2}{m} \mathbb{E}_{\boldsymbol{\sigma}}\left[ \frac{t_1^2}{2}\sum_{i=1}^{m}\frac{1}{2}\lVert \sigma_i \mathbf{x}_i \rVert^2 + \sup_{\mathbf{w}_y \in \mathcal{W}} \frac{1}{2}\lVert \mathbf{w}_y \rVert^2 + \sum_{i=1}^m \frac{1}{2} \langle \nabla \lVert \mathbf{v}_{1,i-1} \rVert^2,\sigma_i \mathbf{x}_i  \rangle\right]\\
	\label{eq:bound-term-(a)-mc-2}
	\leq & Lc^2 \left( \frac{B^2 t_1^2 }{4} + \frac{\alpha \hat{R}_S(h_p)}{m}\right)
\end{align}
Inequality \eqref{eq:bound-term-(a)-mc-1} is obtained by applying Lemma~\ref{lemma:kakade-primal-dual} with $f = \Omega_{\text{aux}}$, and $\mathbf{u} = \mathbf{w}_y$ and $\mathbf{v}_i = t_1\mathbf{x}_i$. And the last step in~\eqref{eq:bound-term-(a)-mc-2} holds due to the fact that $\sup_{W \in \mathcal{W}}\Omega(W) \leq \alpha \hat{R}_S(h_p)$ and $\sup_{\mathbf{x}\in \mathcal{X}} \lVert \mathbf{x}\rVert \leq B$. By dividing $t_1$ on both sides, and notice that the above upper bound holds for any $t_1>0$, we choose a particular $t_1$ making the upper bound tight,
\begin{equation}
	\label{eq:rademacher-term1-mc}
	\frac{Lc^2}{m} \mathbb{E}_{\boldsymbol{\sigma}}\left[\sup_{W\in \mathcal{W}} \sum_{i=1}^m \sigma_i \langle \mathbf{w},\mathbf{x}_i\rangle \right] \leq \inf_{t_1 > 0} Lc^2\left(\frac{B^2 t_1}{4} +\frac{\alpha \hat{R}_S(h_p)}{mt_1}\right) = Lc^2\sqrt{\frac{B^2\alpha \hat{R}_S(h_p)}{m}}.
\end{equation}

Similarly, we introduce variable $t_2$ as tunning parameter for the \texttt{term (b)} with $\mathbf{u} = \mathbf{w}_y$ and $\mathbf{v}_i = t_2 \mathbf{h}_p(\mathbf{x}_i)$,
\[
	\begin{split}
	& \frac{Lc^2}{m} \mathbb{E}_{\boldsymbol{\sigma}}\left[ \sup_{\boldsymbol{\beta}\in \mathcal{V}}\langle \boldsymbol{\beta},t_2 \sigma_i \mathbf{h}_p(\mathbf{x}_i)\rangle\right]\\
	\leq & \frac{Lc^2}{m} \mathbb{E}_{\boldsymbol{\sigma}}\left[\frac{t_2^2}{2} \sum_{i=1}^m \frac{1}{2}\lVert \sigma_i \mathbf{h}_p(\mathbf{x}_i)\rVert^2 + \sup_{\boldsymbol{\beta} \in \mathcal{V}}\frac{1}{2}\lVert\boldsymbol{\beta}\rVert^2 + \sum_{i=1}^m \frac{1}{2}\langle \nabla\lVert \mathbf{v}_{1:i-1}\rVert^2,\sigma_i\mathbf{h}_p(\mathbf{x}_i) \rangle \right]\\
	\leq & Lc^2 \left(\frac{C^2 t_2^2}{4} +\frac{\rho}{m}\right),
	\end{split}
\]
where the last inequality holds due to the fact that $\sup_{\boldsymbol{\beta}\in \mathcal{V}}\lVert\boldsymbol{\beta}\rVert^2 \leq 2\rho$ and $\sup_{\mathbf{x}\in \mathcal{X}} \lVert \mathbf{h}_p(\mathbf{x})\rVert \leq C$. By dividing $t_2$ on both sides, and notice that the above upper bound holds for any $t_2>0$, we choose a particular $t_2$ making the upper bound tight,
\begin{equation}
	\label{eq:rademacher-term2-mc}
	\frac{Lc^2}{m} \mathbb{E}_{\boldsymbol{\sigma}}\left[\sup_{\boldsymbol{\beta}\in \mathcal{V}} \sum_{i=1}^m \sigma_i \langle \boldsymbol{\beta},\mathbf{h}_p(\mathbf{x}_i)\rangle \right] \leq \inf_{t_2 > 0} Lc^2 \left(\frac{C^2 t_2}{4} +\frac{\rho}{m t_2}\right) = Lc^2\sqrt{\frac{C^2\rho}{m}}.
\end{equation}

Combing~\eqref{eq:rademacher-term1-mc} and~\eqref{eq:rademacher-term2-mc}, we obtain obtain the upper bound for $\hat{\mathfrak{R}}_S(\mathcal{L})$. Notice that the square-root function is concave, by applying the Jensen's inequality w.r.t. the both side, we have the following upper bound for Rademacher complexity of loss function family,

\begin{equation}
	\label{eq:rademacher-bound-loss-multi-class}
	\begin{split}
	\mathfrak{R}_m(\mathcal{L}) &= \mathbb{E}_S[\hat{\mathfrak{R}}_S(\mathcal{L})] \leq Lc^2 \mathbb{E}_S\left[\sqrt{\frac{B^2\alpha \hat{R}_S(h_p)}{m}} + \sqrt{\frac{C^2\rho}{m}}\right] \\
	&\leq Lc^2 \left(\sqrt{\frac{B^2\alpha \mathbb{E}_S\left[\hat{R}_S(h_p)\right]}{m}} + \sqrt{\frac{C^2\rho}{m}}\right) \\
	& = L c^2 \left( \sqrt{\frac{B^2\alpha R(h_p)}{ m}} + \sqrt{\frac{C^2\rho}{m}}\right)\\
	& \leq L c^2 \sqrt{\frac{2(B^2\alpha R_p + C^2\rho)}{m}}.	
	\end{split}	
\end{equation}
The last step holds due to the fact that $\sqrt{a} + \sqrt{b} \leq \sqrt{2(a+b)}$, for any $a,b \geq 0$. Hence, we complete the proof.
\end{proof}
\begin{myRemark}
	Our analysis follows the framework in \cite{book/MIT/mohri2012foundations}, and shows a quadratic dependency on the number of classes $c$. In fact, it can be improved into linear or radical dependency by utilizing Gaussian complexity~\citep{conf/nips/LeiDBK15} or vector-contraction inequality for Rademacher complexities~\citep{conf/alt/Maurer16}, and we will show this in the journal version.
\end{myRemark}

\subsection{Proof of Main Theorem}
\begin{proof}
	Since $\hat{W}$ is the optimal solution of regularized ERM in Eq.~\eqref{eq:biased-regularization-multi-class}, it is apparently better than the choice $\mathbf{0}\in \mathbb{R}^{d\times c}$,
	\[
		\hat{R}_{S}(h_{\hat{W},p}) + \lambda \Omega(\hat{W}) \leq \hat{R}_S(h_p) + \lambda \Omega(\mathbf{0}) = \hat{R}_S(h_p),
	\]
	combining the non-negativity of loss function and regularizer, we know that 
	\[
		\hat{R}_S(h_{\hat{W},p}) \leq \hat{R}(h_p) \quad \text{and} \quad  \Omega(\hat{W}) \leq  \hat{R}(h_p)/\lambda.
	\]
	Similar to the proof in~\eqref{eq:upper-bound-r}, we know that the upper bound of risk can be chosen as $r = R_p$, the risk of previous models on target distribution.

	Thus, we can apply Lemma~\ref{lemma:generaization-bound-loss}, obtaining that for $h_{\hat{W},p}$,
	\begin{equation}
	\label{eq:rademacher-hypothesis-mc}
		\begin{split}
			R(h_{\hat{W},p}) - \hat{R}_S(h_{\hat{W},p}) &\leq R_{\ell}(h_{\hat{W},p}) - \hat{R}_S(h_{\hat{W},p}) \\
			& \leq 2 \mathfrak{R}_m(\mathcal{L}) + \frac{3M\log(1/\delta)}{4m} + 3\sqrt{\frac{(8\mathfrak{R}_m(\mathcal{L})+R_p)M\log(1/\delta)}{4m}}.
		\end{split}		
	\end{equation}

Setting $\alpha = 1/\lambda$, we have
\begin{equation}
	\label{eq:bound-rademacher-L-mc}
	\mathfrak{R}_m(\mathcal{L})  \leq L c^2 \sqrt{\frac{2(B^2 R_p + C^2\lambda\rho)}{\lambda m}}.
\end{equation}

Plugging \eqref{eq:bound-rademacher-L-mc} into \eqref{eq:rademacher-hypothesis-mc}, we complete the proof of statements.

\begin{equation}
	\begin{split}
	R(h_{\hat{W},p}) - \hat{R}_S(h_{\hat{W},p}) &\leq 2L c^2 \sqrt{\frac{2(B^2 R_p + C^2\lambda\rho)}{\lambda m}} + \frac{3M\log(1/\delta)}{4m} \\
	& + 3\sqrt{\frac{\left(8L c^2 \sqrt{\frac{2(B^2 R_p + C^2\lambda\rho)}{\lambda m}}+R_p\right)M\log(1/\delta)}{4m}}.	
	\end{split}
\end{equation}

By some simple derivations and transforms, we complete the proof of main statements.
\end{proof}

\section{More Experimental Results}
\label{sec:more-experiments}
In this section, we provide more experimental results from the following four aspects:
\begin{enumerate}
	\item[(1)] weight concentration: we show that the weight distribution learned by $\mathtt{WeightUpdat}$ largely concentrates on the most relevant previous models.
	\item[(2)] recurring concept drift: we conduct performance comparisons on recurring concept drift datasets, and empirically show the effectiveness of \textsc{Condor} under such circumstance.
	\item[(3)] parameter study: we exhibit that our approach is stable to the parameters.
	\item[(4)] robustness comparisons: we demonstrate the superiority of proposed approach by performing robustness comparisons over 22 datasets. 
\end{enumerate}

\subsection{Weight Concentration}
\label{sec:weight-concentration}
In this paragraph, we present the weight distribution of previous models learned by $\mathtt{WeightUpdate}$, in order to show that the weights largely concentrate on the most relevant previous models. 

The phenomenon is validated on both synthetic dataset \textit{SEA-recur} and the real-world dataset \textit{Emailing list}. We synthesize the recurring concept drift dataset according to the similar rule of original SEA dataset~\citep{conf/kdd/StreetK01}. That is, there are three attributes $x_1,x_2,x_3$, and $0\leq x_i \leq 10.0$. The target concept is determined by $x_1+x_2 \leq b$. Table~\ref{table:recurring-sea} shows the detailed information of concept drift in \textit{SEA-recur} dataset. And Table~\ref{table:recurring-weights-result-sea} shows the weight distribution of seven previous models, we can see that the weights concentrate on $\beta_1, \beta_2$ and $\beta_7$, which are the epochs that share the same concept with epoch $S_8$.

We also examine the phenomenon on real-world dataset  \textit{Emailing list}, whose detailed information of concept drift are shown in Table~\ref{table:recurring-email-list}. We can see that the concept drift happens for every 300 round, and in a recurring manner. Let us consider the epoch in 1,200 - 1,500 (epoch 5), actually, the concept in epoch 5 has emerged previously, i.e., it is same as the concepts in epoch 1 and epoch 3. As we can see in Table~\ref{table:recurring-weights-result}, the weights of four previous model essentially concentrate on the $\beta_1$ and $\beta_4$.

Thus, these experiments verify that the returned weight distribution largely concentrates on those previous models who better fit to current data epoch. Essentially, the weight learned by $\mathtt{WeightUpdate}$ depicts the `reusability' of previous model towards current data.

The weight concentration property plays an important role in making our approach successful. Especially, this is also an evidence to justify why our approach also succeeds in recurring concept drift scenarios, and we will show this in next paragraph.

\begin{table}[!t]
\scriptsize 
\centering
\caption{Detailed information of concept drift in \textit{SEA-recur} dataset. The target concept is determined by $x_1+x_2 \leq b$, where $b$ changes along with the data stream.}
\label{table:recurring-sea}
\resizebox{\textwidth}{!}{
\begin{tabular}{ccccccccc}\toprule
Topic 	 & 1 - 100 & 101 - 200 	& 201 - 300 & 301 - 400 & 401 - 500 & 501 - 600  & 601 - 700   & 701 - 800\\\midrule
$b$ 		 &	10     & 10 		& 20 		& 20 	& 20  & 20 	& 10   & 10 \\ \bottomrule
\end{tabular}}
\end{table}

\begin{table}[!t]
\scriptsize 
\centering
\caption{Demonstration of weights concentration phenomenon on synthetic \textit{SEA-recur} dataset. Consider the epoch in 701 - 800, and following shows weights of previous models $h_1,h_2,\ldots,h_7$.}
\label{table:recurring-weights-result-sea}
\resizebox{\textwidth}{!}{
\begin{tabular}{c|ccccccccccc}\toprule 
Iteration   &    1	 &	2	   &	3		 &	4		& 	5	&	$\ldots$	& 	96 	& 97	&	98			& 99		&	100      \\\midrule
$\beta_1$	&	0.142 &	0.159  &	0.176	 &  0.176	&	0.192	& 	&	0.212 	& 	0.212	&  0.212	&  0.212		&	0.212      \\
$\beta_2$   &	0.142 &	0.159  &	0.176	 &  0.176 	&	0.192	& 	&	0.472 	& 	0.472	&  	0.472	&  0.472	&	0.472 \\
$\beta_3$   &	0.142 &	0.130  &	0.118	 &  0.118	&	0.106	& 	&	4.3E-6 	& 	3.5E-6	&  	3.4E-6	&  2.9E-6		&	2.3E-6      \\
$\beta_4$   &	0.142 &	0.130  &	0.118	 &  0.118 	&	0.106	& 	&	3.2E-7  & 	2.6E-7	&  	2.1E-7	&  1.7E-7	&	1.4E-7 \\
$\beta_5$   &	0.142 &	0.130  &	0.118	 &  0.118 	&	0.106	& 	&	6.4E-6  & 	5.3E-6	&  	5.2E-6	&  4.3E-6	&	3.5E-6 \\
$\beta_6$   &	0.142 &	0.130  &	0.118	 &  0.118  	&	0.106	& 	&	4.3E-6  & 	3.5E-6	&  	3.5E-6	&  2.9E-6	&	2.4E-6 \\
$\beta_7$   &	0.142 &	0.159  &	0.176	 &  0.176 	&	0.192	& 	&	0.316  	& 	0.316 	&  	0.316	&  0.316	&	0.316 \\\bottomrule
\end{tabular}}
\end{table}

\begin{table}[!h]
\scriptsize 
\centering
\caption{Detailed information of concept drift in \textit{Emailing list} dataset. There are in total three different topics,  ``+'' indicates the users are interested in it, while ``$-$'' indicates not interested in it. Say an example, in the first period (time stamp 0-300), users' are only interested in messages of the topic medicine.}
\label{table:recurring-email-list}
\resizebox{0.85\textwidth}{!}{
\begin{tabular}{cccccc}\toprule
Topic 	 & 1 - 300 & 301 - 600 & 601 - 900 & 901 - 1,200  & 1,201 - 1,500\\\midrule
Medicine &	$+$ 	& $-$ 	& $+$ 	& $-$ 	& $+$ \\
Space 	 &	$-$ 	& $+$ 	& $-$ 	& $+$ 	& $-$ \\
Baseball &	$-$ 	& $+$ 	& $-$ 	& $+$ 	& $-$ \\ \bottomrule
\end{tabular}}
\end{table}

\begin{table}[!h]
\scriptsize 
\centering
\caption{Demonstration of weights concentration phenomenon on Email list dataset. Consider the epoch in 1,201 - 1,500, and following reports weights associated with previous models $h_1,h_2,h_3,h_4$.}
\label{table:recurring-weights-result}
\resizebox{\textwidth}{!}{
\begin{tabular}{c|ccccccccccc}\toprule 
Iteration   &    1	 &	2	   &	3		 &  4 		&  5  & $\ldots$	&	296 &  297 & 298			& 299		&	300      \\\midrule
$\beta_1$	&	0.25 &	0.299  &	0.345	 &  0.384 	&  0.416  &   		&	0.5 & 0.5	& 0.5			& 0.5		&	0.5      \\
$\beta_2$   &	0.25 &	0.201  &	0.155	 &  0.116 	&  0.084  &   		&	1.9E-52 	& 1.3E-52	 & 8.5E-53	& 5.7E-53	&	3.8E-53 \\
$\beta_3$   &	0.25 &	0.299  &	0.345	 &  0.384 	&  0.416  & 		&	0.5 		& 0.5 	& 0.5			& 0.5		&	0.5      \\
$\beta_4$   &	0.25 &	0.201  &	0.155	 &  0.116  	&  0.084  & 		&	1.9E-52 & 1.3E-52 & 8.5E-53	& 5.7E-53	&	3.8E-53 \\\bottomrule
\end{tabular}}
\end{table}

\subsection{Recurring Concept Drift}
\label{sec:recurring}
In this paragraph, we conduct the performance comparison on the \emph{recurring concept drift} scenario, a special sub-type of concept drift, in which previous concepts may disappear and then re-appear in the future. Therefore, previous models may be beneficial for future learning. Previous studies show that one needs to specifically consider the recurring structure, otherwise the performance will dramatically drop, even for approaches dealing with gradually evolving concept drift.

\textbf{Datasets.} We adopt two popular real-world datasets with recurring concept drift, i.e., \textit{Email list} and \textit{Spam filtering} datasets~\citep{conf/ecai/KatakisTV08,journals/kais/KatakisTV10,conf/iconip/JaberCT13a}. Both are datasets extracted from email corpus, and concept is decided by users' personal interests, which changes in a recurring manner. Detailed information are included in Section~\ref{sec:real-world-info}.

\textbf{Comparisons.} We compare our proposed approach to approaches in the following four categories, (a) Sliding window based approaches, including $\mathtt{SVM}$-$\mathtt{fix}$ (batch implementation by SVM) and $\mathtt{NB}$-$\mathtt{sw}$ (update only use the data in sliding window based on incremental Naive Bayes). (b) Ensemble based approaches, including $\mathtt{Learn}^{\mathtt{++}}.\mathtt{NSE}$, $\mathtt{DWM}$ and $\mathtt{AddExp}$. (c) Model-Reuse based approaches, including $\mathtt{TIX}$ and $\mathtt{DTEL}$. (d) Recurring approaches, which are specifically designed for recurring concept drift scenarios, including Conceptual Clustering and Prediction ($\mathtt{CCP}$) approach, an ensemble method that handle recurring concept drift via similarity clustering~\citep{journals/kais/KatakisTV10}. Also, we compare with Dynamic Adaptation to Concept Change ($\mathtt{DACC}$) and its adaptive variant $\mathtt{ADACC}$, detecting recurring concept drift based on a new second-order online learning mechanism~\citep{conf/iconip/JaberCT13a}.

Since codes of CCP, DACC and ACACC are not available, we directly use the results reported in their papers, as we use the whole dataset without any random splitting according to their settings. The experimental results are reported in Table~\ref{table:recurring-concept-drift}, we can see that \textsc{Condor} exhibits an encouraging performance on both datasets over three different performance measures. It performs significantly better than general concept drift approaches, and is comparable or even better than those approaches specifically designed for recurring concept drift scenario. 

The effectiveness of \textsc{Condor} in recurring datasets lies in the effect of \emph{weight concentration}, since our approach guarantees the weight concentrates on the best-fit previous models (See Observation~\ref{obser:weight-concentration}).

\begin{table}[!t]
\scriptsize 
\centering
\caption{Performance comparison on recurring concept drift datasets, \textit{Email list} and \textit{Spam filtering}.}
\label{table:recurring-concept-drift}
\resizebox{0.95\textwidth}{!}{
\begin{tabular}{ccccc|ccc}\toprule
\multirow{2}{*}{Category} & \multirow{2}{*}{Approach} & \multicolumn{3}{c|}{\textit{Email list}} & \multicolumn{3}{c}{\textit{Spam filtering}}                          \\
                          & & Accuracy  & Precision & Recall & Accuracy & Precision & \multicolumn{1}{c}{Recall} \\\midrule
\multirow{2}{*}{Window} & $\mathtt{SVM}$-$\mathtt{fix}$ & 71.4 	& 73.7  & 72.1  & 88.1  & 82.0  & 68.5  \\
& $\mathtt{NB}$-$\mathtt{sw}$ 	& 74.7 	& 77.9  & 73.2  & 91.9  & 90.2  & 77.0  \\\noalign{\smallskip}\cline{2-8}\noalign{\smallskip}
\multirow{3}{*}{Ensemble}  & $\mathtt{Learn}^{\mathtt{++}}.\mathtt{NSE}$ &  70.0 & 76.5  &  76.5  &  90.4 &  84.5 & 79.6 \\
& $\mathtt{DWM}$    		& 78.2    & 75.1  & 81.4 	& 91.9	&	89.1 	& 84.0 \\
& $\mathtt{AddExp}$   		& 70.4	  & 68.2  & 71.4  	& 91.3	&	90.0 	& 80.7  \\\noalign{\smallskip}\cline{2-8}\noalign{\smallskip}
\multirow{2}{*}{Model-Reuse} 	& $\mathtt{TIX}$  & 86.2 & 88.2 & 88.2 & 88.5 & 82.3 & 69.3\\
&  $\mathtt{DTEL}$   &   86.2    &  88.2     &   88.2    &   86.3    &   73.4    &   71.4    \\\noalign{\smallskip}\cline{2-8}\noalign{\smallskip}
\multirow{3}{*}{Recurring}&  $\mathtt{CCP}$ 		& 77.5  & 79.7  & 77.6  & 92.3  & 85.7  & 83.9  \\
&  $\mathtt{DACC}$ 	& 76.2  & 73.8  & 75.9  & 94.7  & 95.1  & \textbf{97.8}  \\
&  $\mathtt{ADACC}$	& 77.5 	& 75.2 	& 77.2 	& 94.9 	& \textbf{95.6} 	& 97.6  \\\noalign{\smallskip}\cline{2-8}\noalign{\smallskip}
Ours &  \textsc{Condor}		&  \textbf{95.6}    &  \textbf{93.2}     &   \textbf{99.8}     &  \textbf{95.4}	    &   91.1    & 	90.8	\\ \bottomrule
\end{tabular}}
\end{table}

\subsection{Parameter Study}
In this paragraph, we study the effects of parameters in influencing final accuracy. There are three parameters who play a crucial role in the procedure of \textsc{Condor}, and they are model pool size $K$, regularization coefficient $\lambda$ and step size $\eta$.
\begin{figure}[!h]
    \raggedright
    \subfigure[model pool size $K$]{ \label{fig:para-K}
        \includegraphics[clip, trim=3.6cm 10cm 4.2cm 10.0cm, height=0.35\textwidth]{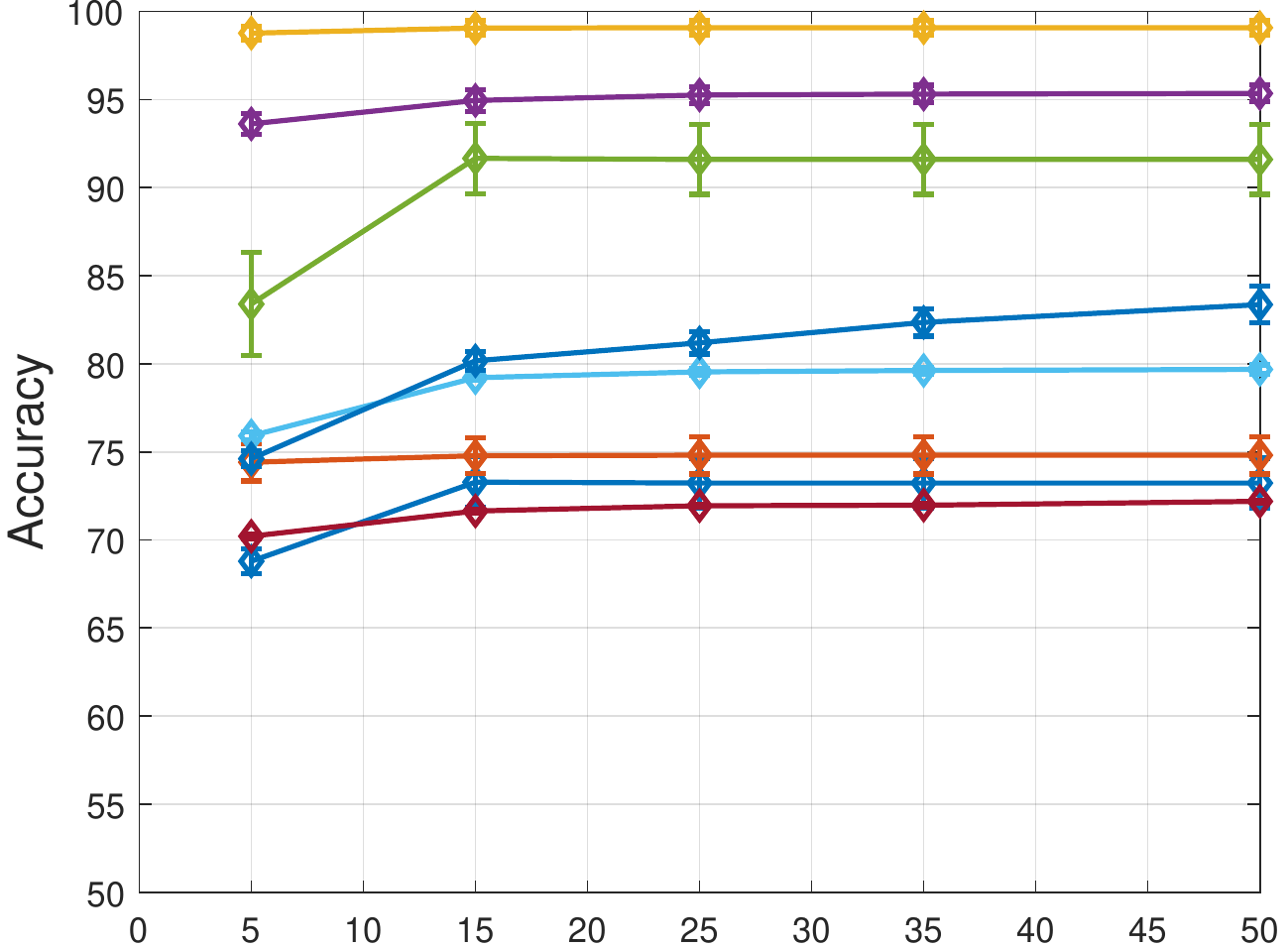}}\hspace{2mm}
    \subfigure[regularization coefficient $\lambda$]{ \label{fig:para-lambda} 
        \includegraphics[clip, trim=3.6cm 10cm 4.2cm 10.0cm, height=0.35\textwidth]{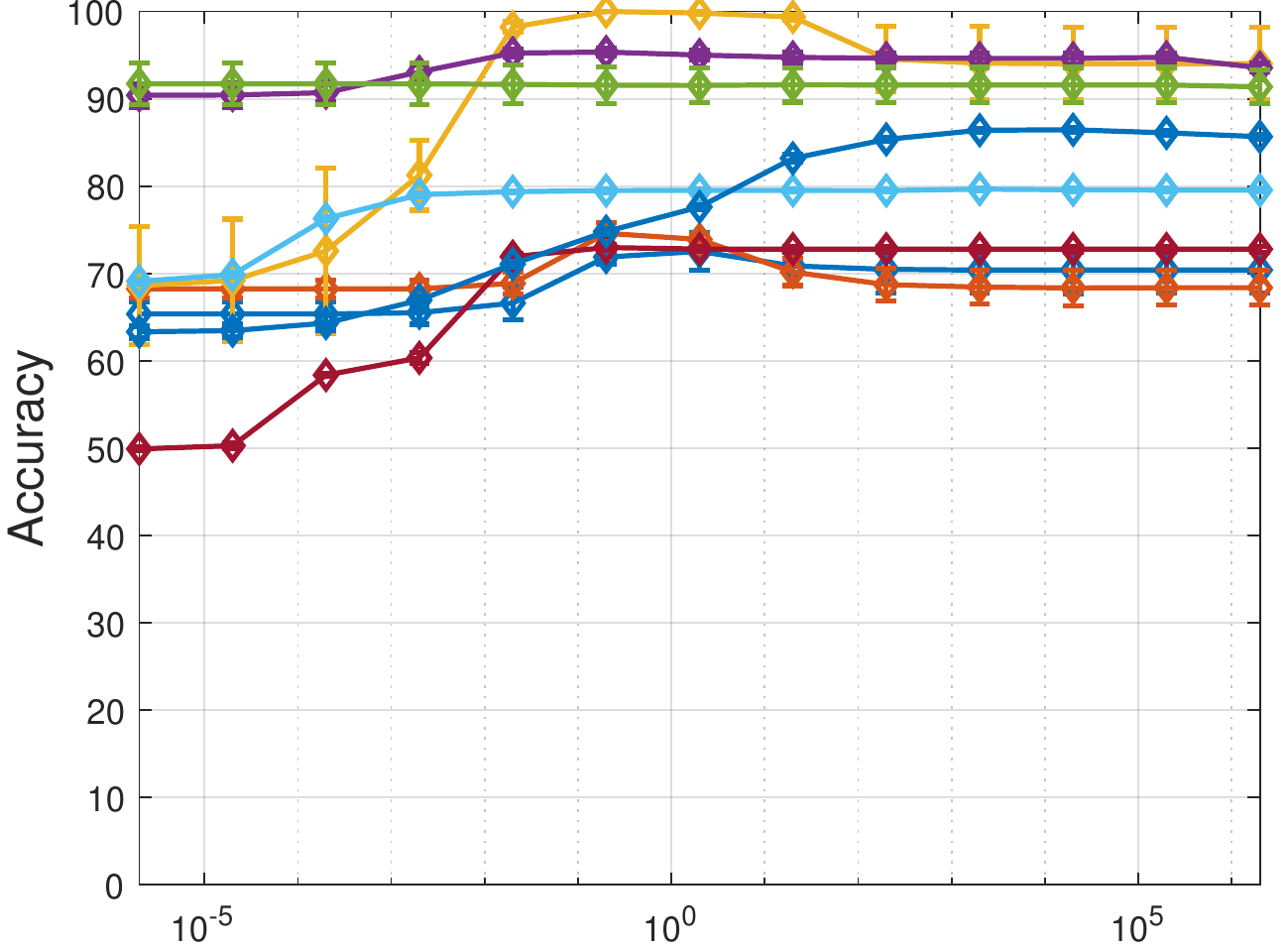}}
    \subfigure[step size $\eta$]{ \label{fig:para-eta} 
        \includegraphics[clip, trim=1.4cm 10cm 6.4cm 10.0cm, height=0.35\textwidth]{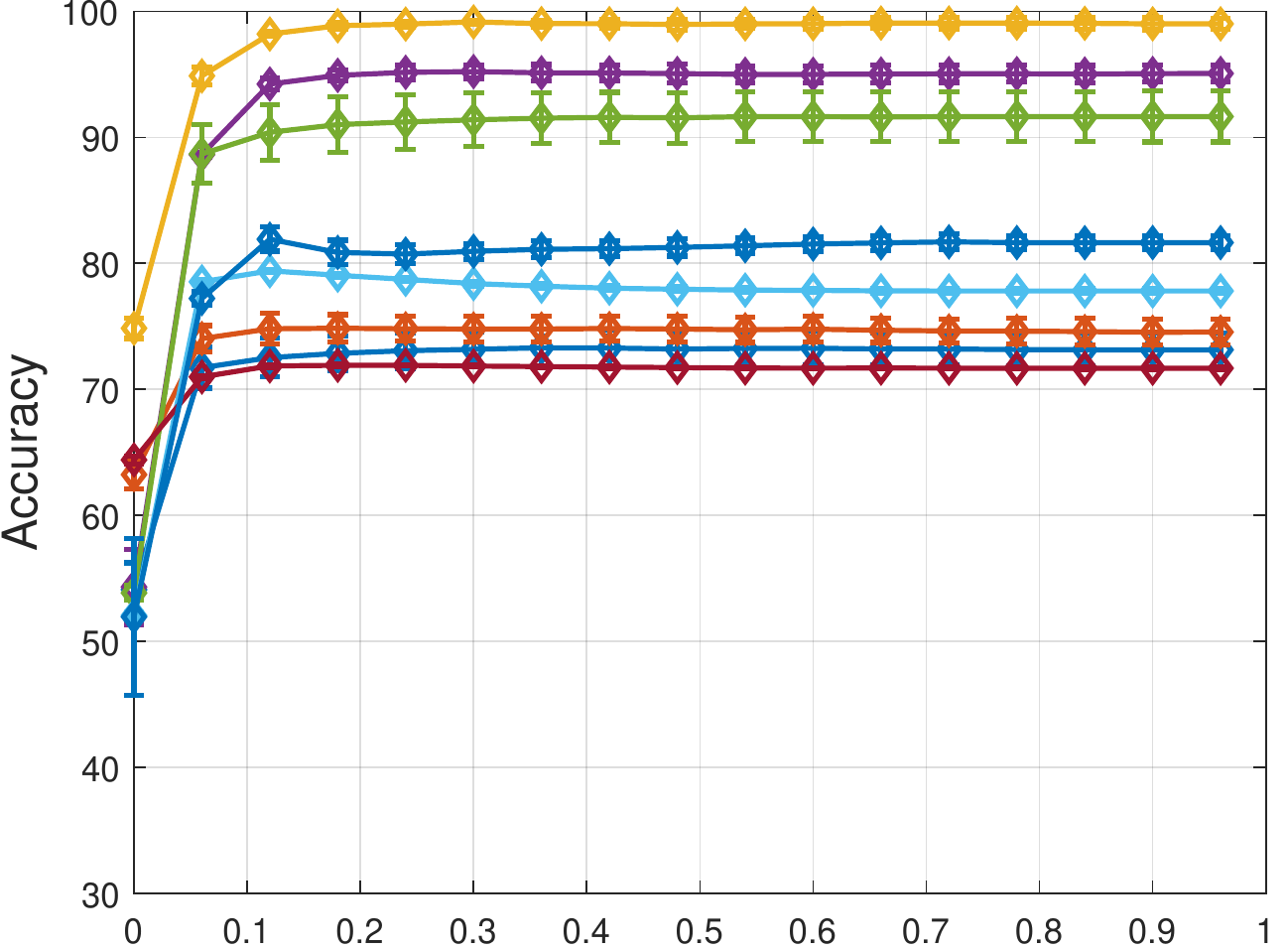}}\hspace{2mm}
    \subfigure{ \label{fig:para-bar} 
        \includegraphics[clip, trim=13.2cm 10cm 1.2cm 10.0cm, height=0.35\textwidth]{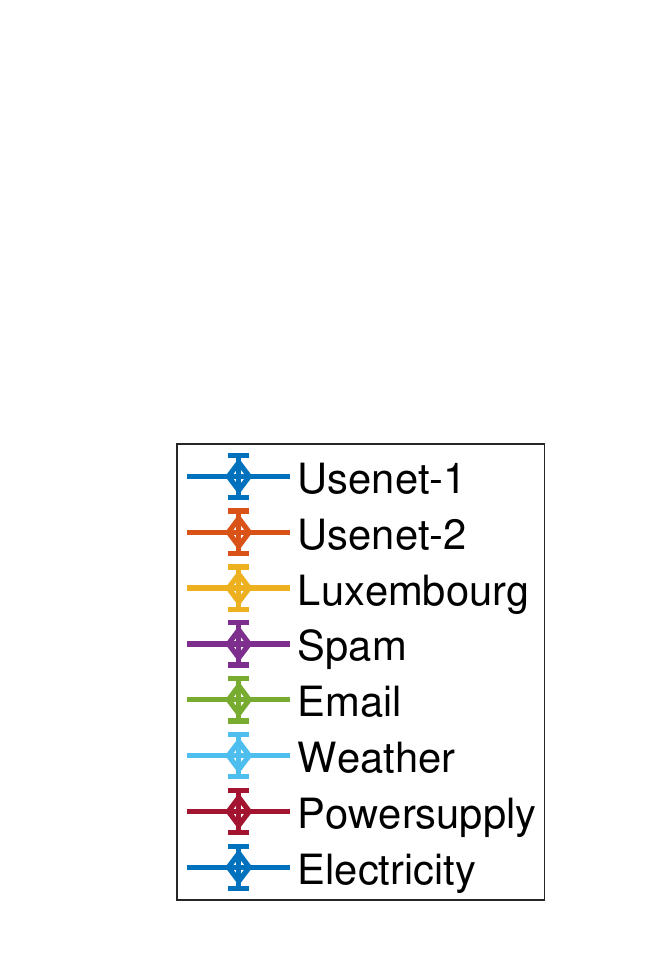}}
             \caption{Parameter study on different datasets.}
    \label{fig:parameter-study}
\end{figure}

\textbf{Model Pool Size.} set the value of model pool size $K$ from $5$ to $50$, and conduct the experiments on all real-world datasets\footnote{Here, we do not include GasSensor (multi-class) and Covertype (extremely large) datasets. Thus, their default settings epoch size is set as $p=200$, different from others, as we have mentioned before.} for 10 times, and plot the mean and standard variance of predictive accuracy with respect to different parameter $K$ values in Figure~\ref{fig:para-K}. We can see that the predictive accuracy rises up as the model pool size increases, and is not benefit from even larger model pool. Although a larger $K$ might be benign for performance improvement, the memory cost will also significantly increase with a larger $K$.
In this paper, we set default value of model pool size $K$ as $25$. 

\textbf{Regularization Coefficient.} We set the value of regularization coefficient $\lambda$ from $2\times10^{-6}$ to $2\times10^{6}$, and conduct the experiments on all real-world datasets for 10 times, and plot the mean and standard variance of predictive accuracy with respect to different parameter $\lambda$ values in Figure~\ref{fig:para-lambda}. We can see that when we set a relative large $\lambda$ value, all datasets basically achieve the best performance, and are not sensitive to the $\lambda$ value. This accords with our intuition, since the $\lambda$ value represents a trade-off between empirical loss and biased regularization term, a larger value addresses more importance on biased regularization. In other words, when $\lambda$ is large, it tends to leverage more information to build the new model. Thus, the results in Figure~\ref{fig:para-lambda} implies the effectiveness of model reuse. In this paper, we set default value of regularization coefficient $\lambda$ as $200$. 

\textbf{Step Size.} We set the value of step size $\eta$ from $0$ to $1$, and conduct the experiments on all real-world datasets for 10 times, and plot the mean and standard variance of predictive accuracy with respect to different parameter $\eta$ values in Figure~\ref{fig:para-eta}. We can see that when step size is set relatively large, say larger than 0.5, then the performance is satisfying and stable. This phenomenon matches the theoretical suggestion value in Lemma~\ref{lemma:local-regret-convex}, as the theoretical suggestion value can be calculated by $\eta_{\texttt{theory}} = \sqrt{8\ln(K/p)} = \sqrt{8\ln(25)/50} \approx 0.718$, where $K$ is the model size default as 25 and $p$ is epoch size default as 50. In this paper, we set default value of step size $\lambda$ as $0.75$.

\subsection{Robustness Comparisons}
Additionally, we conduct robustness comparisons on all approaches over 22 datasets. The robustness measures the performance of a particular approach over all datasets. Concretely speaking, for a particular algorithm \textit{algo}, similar to definition in~\cite{conf/kdd/VlachosDGKK02}, the robustness here is defined as the proportion between its accuracy and the smallest accuracy among all compared algorithms,
\[
r_{algo} = \frac{acc_{algo}}{\min_{\alpha} acc_\alpha}.
\]

Apparently, the worst algorithm has $r_{algo}=1$, and the others have $r_{algo}\geq 1$, the greater the better. Hence, the sum of $r_{algo}$ over all datasets indicates the robustness of for algorithm $algo$. The greater the value of the sum, the better the performance of the algorithm.
\begin{figure}[!t]
\centering
\includegraphics[clip, trim=5.4cm 9.4cm 5.3cm 10cm, width=0.45\textwidth]{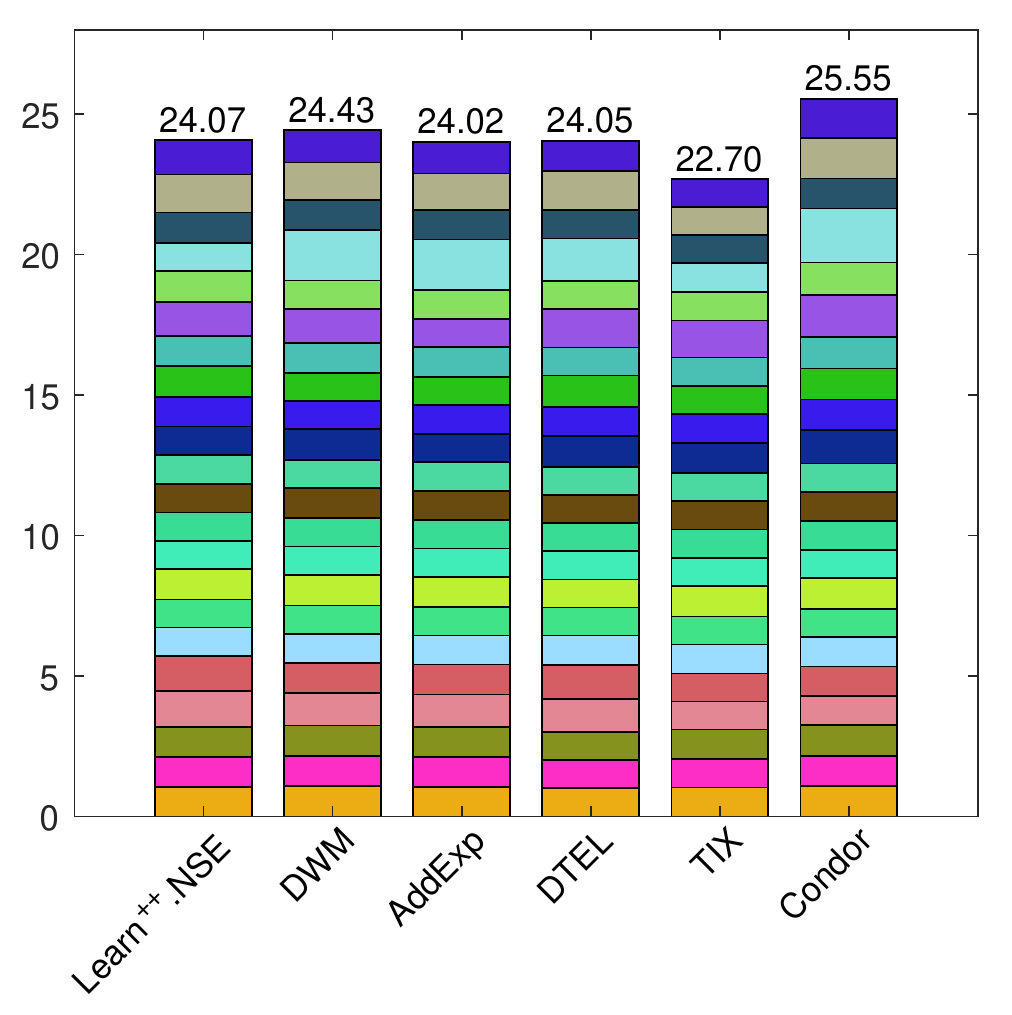}
\caption{Robustness comparisons of accuracy on five compared approaches and \textsc{Condor} over 22 datasets with concept drift.}
\label{figure:robustness}
\end{figure}

We plot the robustness comparison for five compared approaches and \textsc{Condor} over 22 datasets in Figure~\ref{figure:robustness}. From the figure, we can see that \textsc{Condor} achieves the best over all datasets, both ensemble category and model-reuse category approaches. In particular, our approach shows a significant advantage over DTEL and TIX, which verifies the effectiveness of our model reuse strategy. 

\section{Dataset Descriptions}
\label{sec:dataset-description}
In this section, we provide detailed descriptions of datasets with concept drift, which are adopted in the experiments. 

\subsection{Descriptions of Synthetic Datasets}
In the experiments, we adopt four commonly used synthetic datasets and their variants: SEA (SEA200A, SEA200G and SEA500G), CIR500G, SIN500G and STA500G. 

The first family is SEA dataset~\cite{conf/kdd/StreetK01}, which consists of three attributes $x_1,x_2,x_3$, and $0\leq x_i \leq 10.0$. The target concept is determined by $x_1+x_2 \leq b$. For the three variants, there are 24,000 instances. The drift period of SEA200A and SEA200G is 200, and for SEA500G, the drift period is 500. `A' indicates $b \in \mathcal{A}$ and `G' indicates $b\in \mathcal{G}$, where $\mathcal{A}=\{10,7,3,7,10,13,16,13\}$ and $\mathcal{G}= \{10,8,6,8,10,12,14,12\}$.

The other three synthetic datasets are 

CIR500G is a variant of CIRCLE datasets~\citep{journals/tnn/ElwellP11}, which applies a circle as the decision boundary in a 2-D feature space and simulates concept drift by changing the radius of the circle. The target label is $x_1 + x_2^2 \leq r$ with $r = \{3,2.5,2,2.5,3,3.5,4,3.5\}$. The drift period is 500.

SIN500G is a variant of SINE datasets~\citep{journals/tnn/ElwellP11}, which applies a sine curve as the decision boundary in a 2-D feature space and simulates concept drift by changing the angle. The target label is $\sin(x_1+\theta) \leq x_2$ with $\theta_0 = 0$ and $\Delta\theta = \pi/60$. The drift period is 500.

STA500G is a variant of STAGGER Boolean Concepts~\citep{journals/ml/SchlimmerG86}, which generate the data with categorical features using a set of rules to determine the class label. Details is included in~\cite{journals/tnnls/suny18}. The drift period is 500.

The other six synthetic datasets are 1CDT, 1CHT, UG-2C-2D, UG-2C-3D, UG-2C-5D and GEARS-2C-2D. Their basic information are reported in Table~\ref{table:dataset-info}. For more details, one can refer to the paper~\citep{conf/sdm/SouzaSGB15}.

\subsection{Descriptions of Real-World Datasets}
\label{sec:real-world-info}
In the experiments, we adopt nine real-world datasets: Usenet-1, Usenet-2, Luxembourg, Spam, Email, Weather, Powersupply, Electricity and Covertype. The number of data items varies from 1,500 to 581,012. Basic statistics are included in Table~\ref{table:dataset-info}, and we provide detailed descriptions as follows.

\begin{enumerate}[leftmargin=0cm,itemindent=.25cm,labelwidth=\itemindent,labelsep=0cm,align=left]
	\item[-] \textit{Usenet}~\citep{conf/ecai/KatakisTV08} is split into \textit{Usenet-1} and \textit{Usenet-2} which both consist of 1,500 instances with 100 attributes based on 20 newsgroups collection. They simulate a stream of messages from different newsgroups that are sequentially presented to a user, who then labels them according to his/her personal interests. 
	\item[-] \textit{Luxembourg}~\citep{journals/ida/Zliobaite11} is constructed by using European Social Survey data. There are 1,900 instances with 32 attributes in total, and each instance is an individual and attributes are formed from answers to the survey questionnaire. The label indicates high or low internet usage. 
	\item[-] \textit{Spam Filtering}~\citep{journals/jiis/KatakisTBBV09} is a real-world textual dataset that uses email messages from the Spam Assassin Collection, and boolean bag-of-words approach is adopted to represent emails. It consists of 9,324 instances with 500 attributes, and label indicates spam or legitimate.
	\item[-] \textit{Email List}~\citep{journals/jiis/KatakisTBBV09} is a stream of 1,500 examples and 913 attributes which are words that appeared at least 10 times in the corpus (boolean bag-of-words representation), which are collected from 20 Newsgroup collection. The users' personal interests are changing in a recurring manner. 
    \item[-] \textit{Weather}~\citep{journals/tnn/ElwellP11} dataset is originally collected from the Offutt Air Force Base in Bellevue, Nebraska. 18,159 instances are presented with an extensive range of 50 years ($1949-1999$) and diverse weather patterns. Eight features are selected based on their availability, eliminating those with a missing feature rate above 15\%. The remaining missing values are imputed by the mean of features in the preceding and following instances. Class labels are based on the binary indicator(s) provided for each daily reading of rain with 18,159 daily readings: 5698 (31\%) positive (rain) and 12,461 (69\%) negative (no rain).
    \item[-] \textit{GasSensor}~\citep{journal/chemistry/vergara2012chemical} is a dataset contains 4,450 measurements from 16 chemical sensors utilized in simulations for drift compensation in a discrimination task of six gases (six classes) at various levels of concentrations.
    \item[-] \textit{Powersupply}~\citep{journals/archive/chen2015ucr} contains three year power supply records including 29,928 instances with 2 attributes from 1995 to 1998, and our learning task is to predict which hour the current power supply belongs to. We relabel into binary classification according to p.m. or a.m.
    \item[-] \textit{Electricity}~\citep{journal/harries1999splice} is wildly adopted and collected from the Australian New South Wales Electricity Market where prices are affected by demand and supply of the market. The dataset contains 45,312 instances with 8 features. The class label identifies the change of the price relative to a moving average of the last 24 hours. 
    \item[-] \textit{Covertype}~\citep{conf/kdd/GamaRM03,journals/tnnls/suny18} is a real-world data set for describing the observation of a forest area with 51 cartographic variables obtained from US Forest Service (USFS) Region 2 Resource Information System (RIS). Binary class labels are involved to represent the corresponding forest cover type. 
    \item[-] \textit{PokerHand}~\citep{book/cattral2002evolutionary} describes the suits and ranks of a hand of five playing cards, which consists of 1,000,000 instances and 11 attributes. Each card is described using two attributes (suit and rank). Each card is described using two attributes (suit and rank), for a total of 10 predictive attributes. There is one class attribute that describes the ``Poker Hand''.
\end{enumerate}
\end{document}